\pdfoutput=1
\documentclass[11pt]{article} %
\usepackage{times}
\usepackage[letterpaper, left=1in, right=1in, top=1in, bottom=1in]{geometry}

\usepackage[utf8]{inputenc} %
\usepackage[T1]{fontenc}    %
\usepackage{url}            %
\usepackage{booktabs}       %
\usepackage{amsfonts}       %
\usepackage{nicefrac}       %
\usepackage{microtype}      %
\usepackage[colorlinks=true,linkcolor=blue,citecolor=blue]{hyperref}
\usepackage{mkolar_definitions}
\usepackage{xspace}
\usepackage{amsmath}
\usepackage{algorithm,algorithmic}
\usepackage{color}
\usepackage{enumitem}
\usepackage{comment}
\usepackage{bm}
\usepackage{graphicx}
\usepackage{subcaption}
\usepackage[many]{tcolorbox}

\usepackage{apptools}
\usepackage[page, header]{appendix}
\usepackage{titletoc}

\usepackage{mdframed}
\usepackage[scaled=.9]{helvet}

\usepackage{url}
\usepackage{authblk}
\usepackage{caption}

\newcommand{\var}{\operatorname{var}}
\newcommand{\sign}{\operatorname{sign}}

\def\pr{\textnormal{Pr}}
\newcommand{\avg}{\textnormal{avg-diam}}

\newcommand{\pdbal}{\textsc{pdbal}\xspace}
\newcommand{\dbal}{\textsc{dbal}\xspace}
\newcommand{\bald}{\textsc{bald}\xspace}
\newcommand{\qbc}{\textsc{qbc}\xspace}
\newcommand{\bqbc}{\textsc{b-qbc}\xspace}
\newcommand{\random}{\textsc{random}\xspace}
\newcommand{\gaussed}{\textsc{gaussed}\xspace}
\newcommand{\sceig}{\textsc{eig}\xspace}
\newcommand{\varsamp}{\textsc{var}\xspace}
\newcommand{\eig}{EIG\xspace}

\usepackage{natbib}
\bibpunct{(}{)}{;}{a}{,}{,}

\renewcommand{\cite}{\citep}

\makeatletter
\newtheorem*{rep@theorem}{\rep@title}
\newcommand{\newreptheorem}[2]{%
\newenvironment{rep#1}[1]{%
 \def\rep@title{#2 \ref*{##1}}%
 \begin{rep@theorem}}%
 {\end{rep@theorem}}}
\makeatother

\newreptheorem{theorem}{Theorem}
\newreptheorem{lemma}{Lemma}

\newcount\Comments  %
\definecolor{darkgreen}{rgb}{0,0.5,0}
\definecolor{darkred}{rgb}{0.7,0,0}
\definecolor{teal}{rgb}{0.3,0.8,0.8}
\definecolor{orange}{rgb}{1.0,0.5,0.0}
\definecolor{purple}{rgb}{0.8,0.0,0.8}
\newcommand{\kibitz}[2]{\ifnum\Comments=1{\textcolor{#1}{\textsf{\footnotesize #2}}}\fi}

  {\list{}{\leftmargin=#1\rightmargin=#1}\item[]}%
  {\endlist}

\usepackage{authblk}

\title{\textbf{Targeted Active Learning for Probabilistic Models}}
\author[1]{Christopher Tosh}
\author[2]{Mauricio Tec}
\author[1]{Wesley Tansey}
\affil[1]{Memorial Sloan Kettering Cancer Center, New York, NY}
\affil[2]{Harvard University, Cambridge, MA}

\begin{document}

\maketitle

{\def\thefootnote{}
\footnotetext{E-mail:
\texttt{christopher.j.tosh@gmail.com},
\texttt{mauriciogtec@hsph.harvard.edu},
\texttt{tanseyw@mskcc.org}
}}

\vspace{-2em}
\begin{abstract}
A fundamental task in science is to design experiments that yield valuable insights about the system under study.
Mathematically, these insights can be represented as a utility or risk function that shapes the value of conducting each experiment.
We present \pdbal, a targeted active learning method that adaptively designs experiments to maximize scientific utility.
\pdbal takes a user-specified risk function and combines it with a probabilistic model of the experimental outcomes to choose designs that rapidly converge on a high-utility model.
We prove theoretical bounds on the label complexity of \pdbal and provide fast closed-form solutions for designing experiments with common exponential family likelihoods.
In simulation studies, \pdbal consistently outperforms standard untargeted approaches that focus on maximizing expected information gain over the design space.
Finally, we demonstrate the scientific potential of \pdbal through a study on a large cancer drug screen dataset where \pdbal quickly recovers the most efficacious drugs with a small fraction of the total number of experiments.
\end{abstract}

\section{Introduction}

Scientific experiments are often expensive, laborious, and time-consuming to conduct.
In practice, this limits the capacity of many studies to only a small subset of possible experiments.
Limited experimental capacity poses a risk: the sample size may be too small to learn meaningful aspects about the system under study.
However, when experiments can be conducted sequentially or in batches, there is an opportunity to alleviate this risk by adaptively designing each batch.
The hope is that the results of previous experiments can be used to design a maximally-informative batch of experiments to conduct next.

In machine learning, the sequential experimental design task is often posed as an active learning problem.
The active learning paradigm allows a learner to adaptively choose on which data points it wants feedback.
The objective is to fit a high-quality model while spending as little as possible on data collection.
Modern active learning algorithms have shown substantial gains when optimizing models for aggregate objectives, such as accuracy~\citep{ash2021gone} or parameter estimation~\citep{tong2000active}.
Many scientific studies have a more targeted objective than a simple aggregate metric.
For example, one may be interested in assessing the prognostic value of a collection potential biomarkers. 
Accurately modeling the distribution of the biomarker variables may require modeling nuisance variables about the patient, environment, and disease status.
While optimizing an aggregate objective like accuracy can lead to recovery of the parameters of interest, this is merely a surrogate to our true objective.
Consequently, it may lead to a less efficient data collection strategy.

Here we consider the task of \emph{targeted active learning}.
The goal in targeted active learning is to efficiently gather data to produce a model with high utility.
Optimizing data collection for utility, rather than model performance, better aligns the model with the scientific objective of the study.
It can also dramatically reduce the sample complexity of the active learning algorithm.
For instance, in the case of $d$-dimensional linear regression, at least $\Omega(d)$ observations are required to learn the entire parameter vector. However, if the targeted objective is to estimate $k \ll d$ coordinates, there is an active learning strategy that can do so with $O(k)$ queries, provided it is given access to enough unlabeled data (see \pref{app:motive-example} for a formal proof). This toy example shows the potential savings in active learning when the end objective is explicitly taken into account.

We propose Probabilistic Diameter-based Active Learning (\pdbal), a targeted active learning algorithm compatible with any probabilistic model. \pdbal builds on diameter-based active learning~\citep{tosh2017diameter, tosh2020diameter}, a framework that allows a scientist to explicitly encode the targeted objective as a distance function between two hypothetical models of the data. Parts of the model that are not important to the scientific study can be ignored in the distance function, resulting in a targeted distance that directly encodes scientific utility.
\pdbal generalizes \dbal from the simple finite outcome setting (e.g. multiclass classification) to arbitrary probabilistic models, greatly expanding the scope of its applicability.

We provide a theoretical analysis that bounds the number of queries for \pdbal to recover a model that is close to the ground-truth with respect to the target distance. We additionally prove lower bounds showing that under certain conditions, \pdbal is nearly optimal. In a suite of empirical evaluations on synthetic data, \pdbal consistently outperforms untargeted active learning approaches based on expected information gain and variance sampling. 
In a study using real cancer drug data to find drugs with high therapeutic index, 
\pdbal learns a model that accurately detects effective drugs after seeing only 10\% of the total dataset. 
The generality and empirical success of \pdbal suggest it has the potential to significantly increase the scale of modern scientific studies.

\subsection{Related work}

There is a substantial body of work on active learning and Bayesian experimental design. Here, we outline some of the most relevant lines of work.

\paragraph{Bayesian active learning.} The seminal work of \citet{lindley1956measure} introduced expected information gain (\eig) as a measure for the value of a new experiment in a Bayesian context. Roughly, \eig measures the change in the entropy of the posterior distribution of a parameter after conditioning on new data. Inspired by this work, others have proposed maximizing \eig as a Bayesian active learning strategy~\citep{mackay1992information, lawrence2002fast}. Noting that computing entropy in parameter space can be expensive for non-parametric models, \citet{houlsby2011bayesian} rewrite \eig as a mutual information problem over outcomes. Their method, Bayesian active learning by disagreement (\bald), is used for Gaussian process classification. \bald has inspired a large body of work in developing \eig-based active learning strategies, particularly for Bayesian neural networks~\citep{gal2017deep, kirsch2019batchbald}. 
However, despite its popularity, \eig can be shown to be suboptimal for reducing prediction error in general~\citep{freund1997selective}.

One alternative to such information gain strategies is Query by committee (\qbc), \citep{seung1992query, freund1997selective}, which more directly seeks to shrink the parameter space by querying points that elicit maximum disagreement among a committee of predictors. Recently, \citet{riis2022bayesian} applied \qbc to a Bayesian regression setup. Their method, \bqbc, reduces to choosing experiments that maximize the posterior variance of the mean predictor. For the special case of Gaussian models with homoscedastic observation noise, this is equivalent to \eig.

Another Bayesian active learning departure from \eig is the decision-theoretic approach of \citet{fisher2021gaussed}, called \gaussed, based on Bayes' risk minimization. The objective function in \gaussed is similar to the \pdbal objective when in the special case of homoskedastic location models and an untargeted squared error distance function over the entire latent parameter space.

\paragraph{Bayesian optimization.} Black-box function optimization is another classical area of interest in the sequential experimental design literature. In this setting, there is an unknown global utility function that can be queried in a black box manner, and the goal is to find the set of inputs that maximize this function. A standard approach to this problem is to posit a Bayesian non-parametric model (such as a Gaussian process) of the underlying function, and then to adaptively make queries that trade off exploration of uncertainty and exploitation of suspected maxima~\citep{hennig2012entropy, hernandez2015predictive, kandasamy2018parallelised}. One of the key differences between black-box (Bayesian) optimization and targeted (Bayesian) active learning is that in black-box optimization, the underlying utility function is being directly modeled. In targeted active learning, utility may only be indirectly expressed as a function of the underlying probabilistic models. One work that helps to bridge Bayesian optimization and targeted active learning is the decision-theoretic Bayesian optimization framework of \citet{neiswanger2022generalizing}, which considers a richer set of objectives related to the underlying utility function than simply finding a single maximum.

\paragraph{Active learning of probabilistic models.} Beyond the Bayesian methods outlined above, others have considered alternate approaches to active learning for probabilistic models. \citet{sabato2014active} studied active linear regression in the misspecified setting. \citet{agarwal2013selective} designed active learning algorithms for generalized linear models for multiclass classification in a streaming setting. \citet{chaudhuri2015convergence} studied a two-stage active learning procedure for maximum-likelihood estimators for a variety of probabilistic models. \citet{ash2021gone} built on this two stage approach to design an active maximum-likelihood approach for deep learning-based models.

\section{Setting}
\label{sec:setting}

Let $\Xcal$ denote a data space and $\Ycal$ denote a response space. Let $\Dcal$ denote a marginal distribution over $\Xcal$. Our goal is to model our data with some parametric probabilistic model $P_\theta(\cdot ; \cdot)$, where $\theta$ lies in a parameter space $\Theta$ and $P_\theta(y ; x)$ denotes the probability (or density) of observing $y \in \Ycal$ at data point $x \in \Xcal$. We will use the notation $y \sim P_{\theta}(x)$ to denote drawing $y$ from the density $P_\theta(\cdot ; x)$.

We consider models that factorize across data points, that is for $x_1, \ldots, x_n \in \Xcal^n$ and $y_1, \ldots, y_n \in \Ycal^n$, we have 
\begin{align*}
P_{\theta}(y_{1:n} ; x_{1:n}) :=  P_\theta(y_1, \ldots, y_n ; x_1, \ldots, x_n) 
= \prod_{i=1}^n P_\theta(y_i ; x_i). 
\end{align*}
For a data point $x \in \Xcal$ and parameter $\theta \in \Theta$, we denote the entropy of the response to $x$ under model $\theta$ as
\[ H_\theta(x) := \EE_{y \sim P_\theta(x)} \left[ \log \frac{1}{P_\theta(y ; x)} \right].  \]
We will take a Bayesian approach to learning. To that end, let $\pi$ denote a prior distribution over $\Theta$. Given observations $(x_1, y_1), \ldots, (x_n, y_n)$, denote the posterior distribution as
\[ \pi_n(\theta) =  \frac{1}{Z_n} \pi(\theta) \prod_{i=1}^n P_\theta(y_i ; x_i) \]
where $Z_n = \EE_{\theta' \sim \pi}\left[ \prod_{i=1}^n P_{\theta'}(y_i ; x_i) \right]$ is the normalizing constant to make $\pi_n$ integrate to one. In this paper, we will assume that we are in the well-specified Bayesian setting, i.e. there is some ground-truth $\theta^\star \sim \pi$, and when we query point $x_i$, the observation $y_i$ is drawn from $P_{\theta^\star}(\cdot ; x_i)$. We will also use the notation $\pi_n(y; x)$ to denote the posterior predictive density
\[ \pi_n(y; x) = \EE_{\theta \sim \pi_n}\left[ P_\theta(y; x) \right], \]
and the notation $y \sim \pi_n(x)$ to denote drawing $y$ from the density $\pi_n(y; x)$. 

A \emph{risk-aligned distance} is a function $d: \Theta \times \Theta \rightarrow [0,1]$ satisfying two properties for all $\theta, \theta' \in \Theta$:
\begin{itemize}
\item \textbf{Identity}. i.e., $d(\theta, \theta) = 0$.
\item \textbf{Symmetry}. i.e. $d(\theta, \theta') = d(\theta', \theta)$.
\end{itemize}
The requirement that $d(\theta, \theta') \leq 1$ is not onerous -- any smooth distance over a bounded space can be transformed into a distance that satisfies this requirement by rescaling. In our setup, a risk-aligned distance encodes our objective: if we committed to the model $\theta$ when the true model was $\theta^\star$, then we expect to suffer a loss of $d(\theta, \theta^\star)$.

The goal in our setting is to find a posterior distribution $\pi_n$ with small \emph{average diameter}:
\[ \avg(\pi_n) = \EE_{\theta, \theta' \sim \pi_n}[d(\theta, \theta')]. \]
To see that this is a reasonable objective, observe that if $\theta^\star \sim \pi$ and $(x_1, y_1), \ldots, (x_n, y_n)$ is generated according to $P_{\theta^\star}$, then after observing this data, $\theta^\star$ is distributed according to $\pi_n$. If we make predictions by sampling a model from $\pi_n$, the expected risk of this strategy is exactly the average diameter. Moreover, even without this Bayesian assumption, the risk of this strategy can still be bounded above as a function of the average diameter~\citep[Lemma~2]{tosh2017diameter}.

\section{Probabilistic DBAL (\pdbal)}
\label{sec:algorithm}

We first recall the standard (functional) diameter-based active learning algorithm. Let $\Ycal$ be a finite set, and let $\Fcal \subset \{ f: \Xcal \rightarrow \Ycal \}$ denote some function class, let $d(\cdot, \cdot)$ denote a distance over $\Fcal$, and let $\pi_n$ denote a posterior distribution over $\Fcal$. The \dbal approach is to score candidate queries $x \in \Xcal$ according to the function
\begin{equation}
\label{eqn:original-dbal}
v_n(x) =  \max_{y \in \Ycal} \EE_{f, f' \sim \pi_n} \left[ d(f, f') \ind\left[ f(x) = y = f'(x) \right] \right],
\end{equation}
where $\ind[\cdot]$ denotes the indicator function, and then choose the available query $x$ that minimizes $v_n(x)$. To practically implement this, one can sample $f_1, \ldots, f_m \sim \pi_n$ and compute the Monte Carlo approximation
\begin{equation}
\label{eqn:original-mc-dbal}
\hat{v}_n(x)  =  \max_{y \in \Ycal} \frac{1}{{m \choose 2}} \sum_{i < j}  d(f_i, f_j') \ind\left[ f_i(x) = y = f_j'(x) \right] .
\end{equation}
The idea behind \cref{eqn:original-dbal} is that in the realizable setting where the true model is in $\Fcal$, the posterior satisfies
\[ \pi_n(f) \propto \pi_{n-1}(f)  \ind[f(x_n) = y_n].\]
By choosing queries according to \cref{eqn:original-dbal}, \dbal minimizes a function of the diameter of the posterior $\pi_{n+1}$, while also hedging its bets against the worst possible outcome.

When moving from discrete-valued functions to probabilistic models, there are a few issues that arise. The first is that, unless our probabilistic models are deterministic, $P_\theta(\cdot ; x)$ will in general be a distribution over outcomes and not a point mass. Thus, we should not use the indicator function to approximate the posterior update. The second issue is that when our outcomes are continuous, it may be intractable to compute a maximization over potential outcomes.
Finally, even if computation over potential outcomes was not an issue, the outcomes that achieve the maximum may be so unlikely that they should not really be considered at all. Indeed, in the Bayesian setting it only makes sense to consider outcomes that have reasonable probability under $\pi_n$.

\paragraph{Extension to probabilistic models.} With the aim of addressing these issues in mind, we propose the \pdbal objective: for $x \in \Xcal$,
\begin{equation}
\label{eqn:pdbal}
s_n(x) = \EE_{\theta^\star, \theta, \theta' \sim {\pi_n}} \left[ \EE_{y \sim P_{\theta^\star}(x)} \left[d(\theta, \theta') P_\theta(y; x) P_{\theta'}(y; x) e^{2H_{\theta^\star}(x)}\right] \right].
\end{equation}
By using $P_\theta(y; x)$, we are again minimizing some function of the diameter of the posterior $\pi_{n+1}$. Moreover, by switching from a maximization to an expectation over potential outcomes $y$, we avoid the tricky optimization problem. Finally, the entropy term in \cref{eqn:pdbal} balances out the possibility of $\theta^\star$ generating unlikely outcomes. Despite these changes, in \pref{sec:theory} we show that \pdbal enjoys nice optimality properties, even as the scope of its applications has grown considerably over \dbal.

\paragraph{Constant entropy models.} When $\Theta$ parameterizes location models with fixed scale parameters, the entropy term is constant. We rewrite the objective for these models as
\begin{equation}
\label{eqn:pdbal-location}
s_n(x) = \EE_{\theta, \theta' \sim {\pi_n}, y \sim \pi_n(x)} \left[d(\theta, \theta') P_\theta(y; x) P_{\theta'}(y; x)\right].
\end{equation}
For readability, we will discuss approximations of \cref{eqn:pdbal-location}. Extending these ideas to \cref{eqn:pdbal} can easily be done when we can compute $H_{\theta}(x)$ in closed form (as is the case for Gaussians, Laplacians, $t$-distributions, and other common likelihoods) or approximate it sufficiently well.

A general approach to approximating \cref{eqn:pdbal-location} is to draw $\theta_1, \ldots, \theta_m \sim \pi_n$ and $y_i \sim P_{\theta_i}(x)$ for $i=1, \ldots, m$, and use the estimate
\begin{equation}
\label{eqn:mc-pdbal}
\hat{s}_n(x) = \frac{1}{{m \choose 3}} \sum_{i < j < k} d(\theta_i, \theta_j) P_{\theta_i}(y_k; x) P_{\theta_j}(y_k; x).
\end{equation}
Although \cref{eqn:mc-pdbal} is an unbiased estimator, it may have large variance due to the sampling $y_i \sim P_{\theta_i}(x)$. In some cases, we can reduce this variance by avoiding sampling the $y_i$'s whenever we can compute the function
\[ M(x; \theta_1, \theta_2, \theta_3) = \EE_{y \sim P_{\theta_1}(x)}\left[  P_{\theta_2}(y; x) P_{\theta_3}(y; x) \right]. \]
This allows us to compute the alternate approximation
\begin{equation}
\label{eqn:alt-mc-pdbal}
\hat{s}_n(x) \ = \ \frac{1}{{m \choose 3}} \sum_{i < j < k} d(\theta_i, \theta_j)  M(x; \theta_i, \theta_j, \theta_k).
\end{equation}
Computing the sums in \cref{eqn:mc-pdbal} or \cref{eqn:alt-mc-pdbal} would take $O(m^3)$ time. In practice, we approximate these sums via Monte Carlo by subsampling $N_{\text{mc}}$ triples $(i, j, k)$. Thus, to compute our approximation of \cref{eqn:pdbal} for a set of $B$ potential queries takes time $O(B N_{\text{mc}} )$. \pref{alg:pdbal} presents the full \pdbal selection procedure.

The following proposition shows that for the important case of Gaussian likelihoods, we can indeed compute the function $M$ in closed form.
\begin{proposition}
\label{prop:gaussian-shortcut}
Fix $d \geq 1$, and let $\mu_i \in \RR^d$ and $\sigma_i^2 >0$ for $i=1,2,3$.
\begin{align*}
 \EE_{y \sim \Ncal(\mu_1, \sigma_1^2 I_d)} \left[ \Ncal(y; \mu_2, \sigma_2^2 I_d) \Ncal(y; \mu_3, \sigma_3^2 I_d) \right] 
 = \left( \frac{1}{ \alpha (2 \pi)^2} \right)^{d/2} \exp \left( - \frac{\sigma_{1}^2\sigma_{2}^2\sigma_{3}^2}{2 \alpha^2}  \sum_{i\neq j \neq k} \sigma_{k}^2 \| \mu_{i} - \mu_{j} \|^2 \right),
\end{align*}
where $\alpha = \sigma_{1}^2 \sigma_{2}^2 +  \sigma_{2}^2 \sigma_{3}^2 +  \sigma_{1}^2 \sigma_{3}^2$.
\end{proposition}
All proofs are deferred to the appendix. In \pref{app:alg-proofs}, we work out closed form solutions for other important likelihoods, including multinomial and exponential.

\begin{algorithm}[t]
\caption{\pdbal selection}
\label{alg:pdbal}
\begin{algorithmic}
\REQUIRE Candidate queries $x_1, \ldots, x_B \in \Xcal$, posterior distribution $\pi_n$, Monte Carlo parameters $m, N_{\text{mc}}$.
\ENSURE Next query $x_b$.
\STATE Draw $\theta_1, \ldots, \theta_m \sim \pi_n$.
\STATE Draw $(i_1, j_1, k_1), \ldots, (i_{N_{\text{mc}}}, j_{N_{\text{mc}}}, k_{N_{\text{mc}}})$ uniformly from the set $\{ (i, j, k) : 1 \leq i < j < k \leq m \}$.
\FOR{$b=1, \ldots, B$}
	\IF{$M$ computable in closed form.}
		\STATE Compute
\begin{align*}
	\hat{s}_n(x_b) =  \frac{1}{N_{\text{mc}}} \sum_{t=1}^{N_{\text{mc}}} d(\theta_{i_t}, \theta_{j_t})  M(x_b; \theta_{i_t}, \theta_{j_t}, \theta_{k_t}).
\end{align*}
	\ELSE
		\STATE Draw $y^{(b)}_1 \sim P_{\theta_1}(x_b), \ldots, y^{(b)}_m \sim P_{\theta_m}(x_b)$.
		\STATE Compute
\begin{align*}
	 \hspace{-1.5em}\hat{s}_n(x_b) =   \frac{1}{N_{\text{mc}}} \sum_{t=1}^{N_{\text{mc}}}  d(\theta_{i_t}, \theta_{j_t})  P_{\theta_{i_t}}(y^{(b)}_{k_t}; x_b) P_{\theta_{j_t}}(y^{(b)}_{k_t}; x_b).
\end{align*}
	\ENDIF
\ENDFOR
\RETURN $\argmin\limits_{x_b} \hat{s}_n(x_b)$.
\end{algorithmic}
\end{algorithm}

\section{Theory}
\label{sec:theory}

For the purposes of this section, we will assume that all models induce the same entropy for a given $x$. 
\begin{assum}
\label{assum:homog-entropy}
Given any $x \in \Xcal$, there is a value $H(x)$ such that $H_\theta(x) = H(x)$ for all $\theta \in \Theta$.
\end{assum}
\pref{assum:homog-entropy} is satisfied, for example, whenever $P_{\theta}$ is a location model whose scale component is fixed or otherwise assumed to be independent of $\theta$. \pref{assum:homog-entropy} is not required to implement \pdbal, but rather only factors into the analysis in this section.

For a prior $\pi$, observed data $(x_{1:n}, y_{1:n})$, and value $\rho \in [0,1]$, we say that a data point $x \in \Xcal$ \emph{$\rho$-splits} the posterior $\pi_n$ if
\begin{equation}
\label{eqn:splitting}
s_n(x) \leq (1- \rho) \avg(\pi_n),
\end{equation}
where $s_n$ is the objective function defined in \cref{eqn:pdbal}. Intuitively, \cref{eqn:splitting} captures the notion that there exists a query $x$ that shrinks the average diameter by at least $(1-\rho)$ in expectation.

We say that $\pi_n$ is \emph{$(\rho,\tau)$-splittable} if 
\begin{equation}
\label{eqn:rho-tau-splittable}
\pr_{x \sim \Dcal}( x \text{ $\rho$-splits $\pi$}) \geq \tau.
\end{equation}
Then for parameters $\rho, \epsilon, \tau \in (0,1)$, we say that $\Theta$ (along with corresponding marginal distribution $\Dcal$) has \emph{splitting index $(\rho, \epsilon, \tau)$} if for every posterior $\pi_n$ over $\Theta$ satisfying $\avg(\pi_n) > \epsilon$, $\pi_n$ is $(\rho, \tau)$-splittable. The definition of splitting in \cref{eqn:splitting} is similar to those provided in previous diameter-based active learning works~\citep{tosh2017diameter, tosh2020diameter}. It corresponds to a requirement that posteriors which are not too concentrated ($\avg > \epsilon$) should have a reasonable number of good queries (at least $\tau\%$ if sampled from $\Dcal$).

One notable difference in this definition is the entropy term inside $s_n(x)$ in \cref{eqn:splitting}.
Without this term, a query with a noisy likelihood will be given higher saliency simply because it produces a wide range of possible outcomes. The entropy term balances out this bias by penalizing queries with a low signal-to-noise ratio.

A key observation in our analysis is that if we query a point that $\rho$-splits the current posterior $\pi_t$, then in expectation a certain potential function will decrease.
\begin{lemma}
\label{lem:splitting}
If we query a point $x_{t+1}$ that $\rho$-splits $\pi_t$, then
\[ \EE_{y_{t+1}}\left[ W_{t+1}^2 \avg(\pi_{t+1}) \right] \leq (1- \rho) W_t^2 \avg(\pi_{t}), \]
where $W_t = e^{\sum_{i=1}^t H(x_{i})} \EE_{\theta \sim \pi} \left[ \prod_{i=1}^t P_{\theta}(y_i ; x_i) \right]$.
\end{lemma}

\subsection{Terminology}

In order to prove bounds on the performance of \pdbal, we will need to make some assumptions about the complexity of the class $\Theta$ and the rates at which empirical entropies within this class concentrate. For a sequence of data pairs $\omega_n = ((x_1, y_1), \ldots, (x_n, y_n))$, let $\Theta|_{\omega_n}$ denote the projection of $\Theta$ onto $\omega_n$. That is:
\[ \Theta|_{\omega_n} = \{ (P_\theta(y_1 ; x_1), \ldots, P_{\theta}(y_n ; x_n) \, : \, \theta \in \Theta  \} . \]
For a sequence $\omega_n$ and parameter $\epsilon > 0$, define $N(\epsilon, \Theta|_{\omega_n}, d_{ll})$ as the size of the minimum cover of $\Theta|_{\omega_n}$ with respect to the distance
\[ d_{ll}(a,b) =  \left| \sum_{i=1}^n \log \frac{a_i}{b_i} \right| \ = \ \left| \sum_{i=1}^n \log {a_i} - \log {b_i} \right|. \]
Here, we consider $\log\frac{0}{0} = 0$. The uniform covering number $N_{ll}(\epsilon, \Theta, n)$ is given by
\[ N_{ll}(\epsilon, \Theta, n) = \max \left\{ N(\epsilon, \Theta|_{\omega_n}, d_{ll}) \, : \, \omega_n \in (\Xcal \times \Ycal)^n  \right\}. \]
We say that a class $\Theta$ is has log-likelihood dimension $(c, d)$ if 
$ \log N_{ll}(\epsilon, \Theta, n)  \leq d \log \left( \frac{c n}{\epsilon} \right)$
for $n \geq d$. The definition of $N_{ll}(\epsilon, \Theta, n)$ is exactly the same as other uniform covering numbers~\cite{anthony1999neural}, modulo the non-standard distance. In the language of statistical learning theory, the log-likelihood dimension is a bound on the \emph{metric entropy} $\log N_{ll}(\epsilon, \Theta, n)$. We will assume that the log-likelihood dimension is bounded.

\begin{assum}
\label{assum:lldim-bound}
$\Theta$ has log-likelihood dimension $(c,d)$,
\end{assum}

As an example of a class with bounded log-likehood dimension, the following result shows how the complexity of a function class translates to the log-likehood dimension of the corresponding Gaussian location model.
\begin{proposition}
\label{prop:gaussian-dimension}
Fix $\sigma^2, d, B > 0$. Let $\Theta$ denote the class of Gaussian location models parameterized by a set of mean functions $ \Fcal \subset \{ \theta: \Xcal \rightarrow [-B, B] \}$ such that $P_\theta(y; x) = \Ncal(y \mid \theta(x), \sigma^2)$. If the responses $y$ lie in $[-B,B]$ and the pseudo-dimension of $\Fcal$ is bounded by $d$, then $\Theta$ has log-likelihood dimension $(cB^2/\sigma^2, d)$ for some universal constant $c>0$.
\end{proposition}

Recall a mean-zero random variable $X$ is \emph{sub-Gamma} with variance factor $v > 0$ and scale parameter $c > 0$ if
\[ \log \EE\left[ e^{\lambda X} \right] \leq \frac{\lambda^2 v}{2(1-c \lambda)}  \]
for all $\lambda \in (0, 1/c)$. We say that the class $\Theta$ is \emph{entropy sub-Gamma} with variance factor $v > 0$ and scale parameter $c > 0$ if the random variable $X = \log \frac{1}{P_{\theta}(Y; x)} - H(x)$ is sub-Gaussian with variance factor $v > 0$ for all $\theta \in \Theta$ and $x \in \Xcal$, where $Y \sim P_{\theta}(x)$. For our bounds we will need $\Theta$ to be entropy sub-Gamma.

\begin{assum}
\label{assum:entropy-subgamma-bound}
$\Theta$ is entropy sub-Gamma with variance $v$ and scale $c'$,
\end{assum}

Many likelihoods satisfy \Cref{assum:entropy-subgamma-bound} including Gaussians, as illustrated by the following result.

\begin{proposition}
\label{prop:gaussian-entropy}
Fix $\sigma^2 > 0$ and let $\Theta$ denote the class of Gaussian location models from \pref{prop:gaussian-dimension}. Then $\Theta$ is entropy sub-Gamma with variance factor 1 and scale parameter 1.
\end{proposition}

Finally, we will require boundedness of both the entropy and the densities of models in $\Theta$.

\begin{assum}
\label{assum:model-bound}
There are constants $c_1, c_2 \geq 0$ such that $P_\theta(y; x) \leq c_1$ and $\exp(H(x)) \leq c_2$ for all $\theta \in \Theta$, $x\in \Xcal$, and $y \in \Ycal$.
\end{assum}

\subsection{Upper bounds}

Given the terminology above, we have the following guarantee on \pdbal.
\begin{theorem}
\label{thm:dbal-upper-bound}
Pick $d \geq 4$ and suppose \Crefrange{assum:homog-entropy}{assum:model-bound} hold. If at every round $t$ we make a query that $\rho$-splits $\pi_t$ and terminate when $\avg(\pi_t) \leq \epsilon$, then with probability $1-\delta$, \pdbal terminates after fewer than 
\begin{align*}
 T &\leq  O \left( \max \left\{ 
\left( \frac{c_1 c_2}{\rho} \right)^2 v \log\left( \frac{c_1 c_2}{\rho \delta} \right),
\frac{d + c'}{\rho}\log\left( \frac{(d + c')c}{\rho \delta} \right), 
\frac{1}{\rho} \log \frac{\avg(\pi)}{\epsilon} 
\right\} \right) 
\end{align*}
queries with a posterior satisfying $\avg(\pi_T) \leq \epsilon$.
\end{theorem}
As discussed above, \Crefrange{assum:homog-entropy}{assum:model-bound} are conditions on the complexity and form of $\Theta$. The requirement on the behavior of \pdbal can be guaranteed with high probability across all rounds $t$ with enough unlabeled data and a fine enough Monte Carlo approximation of \cref{eqn:pdbal}.

Given \pref{lem:splitting}, the proof of~\pref{thm:dbal-upper-bound} takes two steps:
\begin{enumerate}
	\item Showing that $W_t^2 \avg(\pi_{t})$ must decrease exponentially quickly.
	\item Showing that $W_t$ cannot decrease too quickly.
\end{enumerate}
The only way that both of these can hold is if $\avg(\pi_{t})$ must also decrease quickly, proving \pref{thm:dbal-upper-bound}.

The first step, showing that $W_t^2 \avg(\pi_{t})$ decreases quickly, is formalized by the following lemma.
\begin{lemma}
\label{lem:X_t-contraction}
Suppose \pref{assum:model-bound} holds. For any $t \geq 1$ and $\delta > 0$, if $x_i$ $\rho$-splits $\pi_{i-1}$ for $i=1, \ldots, t$, then with probability at least $1-\delta$,
\[ {W_{t}^2 \avg(\pi_t)}  \leq \avg(\pi) \exp\left(- t\rho + c_1 c_2 \sqrt{2t \log \frac{1}{\delta}} \right) . \]
\end{lemma}

The second step, showing that $W_t$ does not decrease too quickly, is formalized by the following lemma.
\begin{lemma}
\label{lem:lower-bound-W_t}
Suppose \pref{assum:lldim-bound} and \pref{assum:entropy-subgamma-bound}. If $\theta^\star \sim \pi$ and $y_i \sim P_{\theta^\star}(\cdot; x_i)$ for $i=1,\ldots, t$, then with probability at least $1-\delta$
\[ W_t \geq \exp\left(-2 - d \log(c t) - \sqrt{ 2 t v \log \frac{2}{\delta}} -  (c'+1) \log \frac{2}{\delta} \right).\]
\end{lemma}

\subsection{Lower bounds}

We now turn to showing that in some cases, any optimal active learning strategy must have some dependence on the splitting index of the class. Our first result along these lines is in the deterministic setting.

\begin{theorem}
\label{thm:deterministic-lower-bound}
Let ${\Theta}$ denote a class of deterministic models that is not $(\rho, \epsilon, \tau)$-splittable for some $\rho, \epsilon \in (0,1/4)$ and $\tau \in (0,1/2)$. Let $\pi$ be any prior distribution satisfying $\avg(\pi) \geq 4 \epsilon$ which is not $(\rho, \tau)$-splittable. Then any active learning strategy that, with probability at least $5/6$ (over the random samples from $\Dcal$ and the observed responses), finds a posterior distribution satisfying $\avg(\pi_t) \leq \epsilon$ must either observe at least $\frac{1}{2\tau}$ unlabeled data points or make at least $\frac{1}{2\rho}$ queries.
\end{theorem}

We can relax the constraint that our models are deterministic to the case where they have bounded entropy at the expense of a slightly weaker lower bound.
\begin{theorem}
\label{thm:lowent-lower-bound}
Let ${\Theta}$ denote a class of models that is not $(\rho, \epsilon, \tau)$-splittable for some $\rho, \epsilon \in (0,1/4)$ and $\tau \in (0,1/2)$ such that $H(x) = h < \rho^{3/2}/6$ and $P_\theta(y;x) \leq 1$ for all $x \in \Xcal$, $\theta \in \Theta$ and $y \in \Ycal$. Let $\pi$ be any prior distribution satisfying $\avg(\pi) \geq 4 \epsilon$ which is not $(\rho, \tau)$-splittable. Then any active learning strategy that, with probability at least $5/6$ (over the random samples from $\Dcal$ and the observed responses), finds a posterior distribution satisfying $\avg(\pi_t) \leq \epsilon$ must either observe at least $\frac{1}{2\tau}$ unlabeled data points or make at least $\frac{1}{2\sqrt{\rho}}$ queries.
\end{theorem}

The key ingredient to proving these lower bounds is demonstrating that splitting values are (approximately) sub-additive.

\begin{lemma}
\label{lem:combine-split}
Let $\rho_1, \rho_2$ satisfy $\rho_1 + \rho_2 < 1$. Suppose $x_1$ $\rho_1$-splits $\pi$, $x_2$ $\rho_2$-splits $\pi$, and $H(x_1) = H(x_2) = h$. Then the following holds:
\begin{itemize}
	\item If $h = 0$, then the combined query $(x_1, x_2)$ has splitting value at most $\rho_1 + \rho_2$. 
	\item If $0 \leq h < \frac{\rho_1 + \rho_2}{6}$, then the combined query $(x_1, x_2)$ has splitting value at most $2(\rho_1 + \rho_2)$. 
\end{itemize} 
\end{lemma}

\section{Empirical results}
\label{sec:empirical}

In this section, we present our results on a suite of synthetic regression simulations as well as the results of a real-data study on a cancer drug discovery problem.

\subsection{Synthetic regression simulations}
\label{subsec:empirical:synth}

\paragraph{Probabilistic models.}  We evaluated \pdbal on several probabilistic regression models.\footnote{Code for our simulations can be found at: \texttt{https://github.com/tansey-lab/pdbal}.} In each of our regression experiments, the model was parameterized by a coefficient vector $\theta \in \RR^d$. Presented in this section are our results for linear regression with homoscedastic Gaussian noise, logistic regression, Poisson regression with the exponential link function, and Beta regression using the mean parameterization~\citep{ferrari2004beta}:
\[ P_\theta(y ; x) = \text{Beta} \left( y \mid  \phi \mu, \phi(1- \mu) \right), \]
where $\mu = \frac{1}{1 + e^{-\langle x, \theta \rangle}}$, $x \in \RR^d$ is the feature vector, and $\phi>0$ is a fixed constant.

For all experiments, we used a normal prior distribution on $\theta$ with identity covariance. For the linear regression setting, the posterior can be computed in closed form. We implemented the other models in PyStan and sampled from the posteriors using the No-U-Turn Sampler (NUTS)~\citep{hoffman2014nuts, pystan, stan}.

\paragraph{Objectives and distances.} We considered three objectives with corresponding distance measures over parameters. 
\begin{itemize}
	\item[(i)] First coordinate sign identification with corresponding distance
\[ d_{\text{first}}( \theta, \theta' ) = \ind[ \sign(\theta_1) \neq \sign(\theta'_1)]. \]
	\item[(ii)] Maximum magnitude coordinate identification:
\[ d_{\text{max}}( \theta, \theta' ) = \ind[ \argmax_{i} |\theta_i| \neq \argmax_{i} |\theta'_i| ]. \]
	\item[(iii)] Coordinate magnitude ranking:
\[ d_{\text{kendall}}( \theta, \theta' ) =  \frac{1}{2} \left( 1 - \tau(|\theta|, |\theta'|)  \right) \]
where $\tau(|\theta|, |\theta'|)$ is Kendall's $\tau$ correlation between of the pairs $(|\theta_1|, |\theta_1'|), \ldots, (|\theta_d|, |\theta_d'|)$.
\end{itemize}
In \pref{app:more-simul}, we present results on more settings.%

\paragraph{Baseline comparisons.} We compared against three baselines: \random, \varsamp, and \sceig. \random chooses its queries uniformly at random from the set of available queries. The learning curve of \random mimics what we would expect to see in a standard passive learning setting. \varsamp is the variance sampling strategy -- it chooses queries based on maximizing the posterior predictive variance:
\[ \var_{y \sim \pi_n(x)}(y) = \EE_{y \sim \pi_n(x)}[y^2] - \EE_{y \sim \pi_n(x)}[y]^2. \]
The law of total variance allows us to rewrite this objective as
\[ \var_{y \sim \pi_n(x)}(y) = \EE_{\theta \sim \pi_n(x)} \left[ \var_{y \sim P_\theta(x)}(y) \right] + \var_{\theta \sim \pi_n(x)} \left( \EE_{y \sim P_\theta(x)}[y] \right) . \]
For all the likelihoods we consider in this section, both $\var_{y \sim P_\theta(x)}(y)$ and $ \EE_{y \sim P_\theta(x)}[y]$ can be computed in closed form.

\sceig is the expected information gain strategy. We use the BALD formulation of \sceig~\citep{houlsby2011bayesian} which chooses queries based on maximizing the mutual information between the outcome and the latent parameter $\theta$:
\begin{align*}
 \Ical(y; \theta | x, \pi_n) &=  H_{\pi_n}(x) - \EE_{\theta \sim \pi_n} \left[ H_\theta(x) ]  \right],
\end{align*}
where $H_{\pi_n}(x)$ is the entropy of the posterior predictive $\pi_n(x)$. In some cases, such as linear regression, this can be computed in closed form, as it is proportional to the posterior predictive variance. For our other settings, we approximate it via numerical integration. 

In our experiments, the ground truth $\theta^\star$ was drawn uniformly from vectors of length 2. The data points were drawn from a mixture distribution: with probability $1-p$ it is drawn uniformly from vectors of length 1, and with probability $p$ each coordinate is set to 0 with probability $1/d$ and the remaining coordinates are drawn so that the vector is of length 1. For some objectives, this sparse distribution contains rare but informative data points. For all of our simulations, the data dimension was set to $d=10$ and the mixing proportion was set to $p=1/10$. All simulations were performed with 250 random seeds, and 95\% confidence intervals are depicted using shading in our plots.

\begin{figure}
\begin{center}
\includegraphics[width=0.9\textwidth]{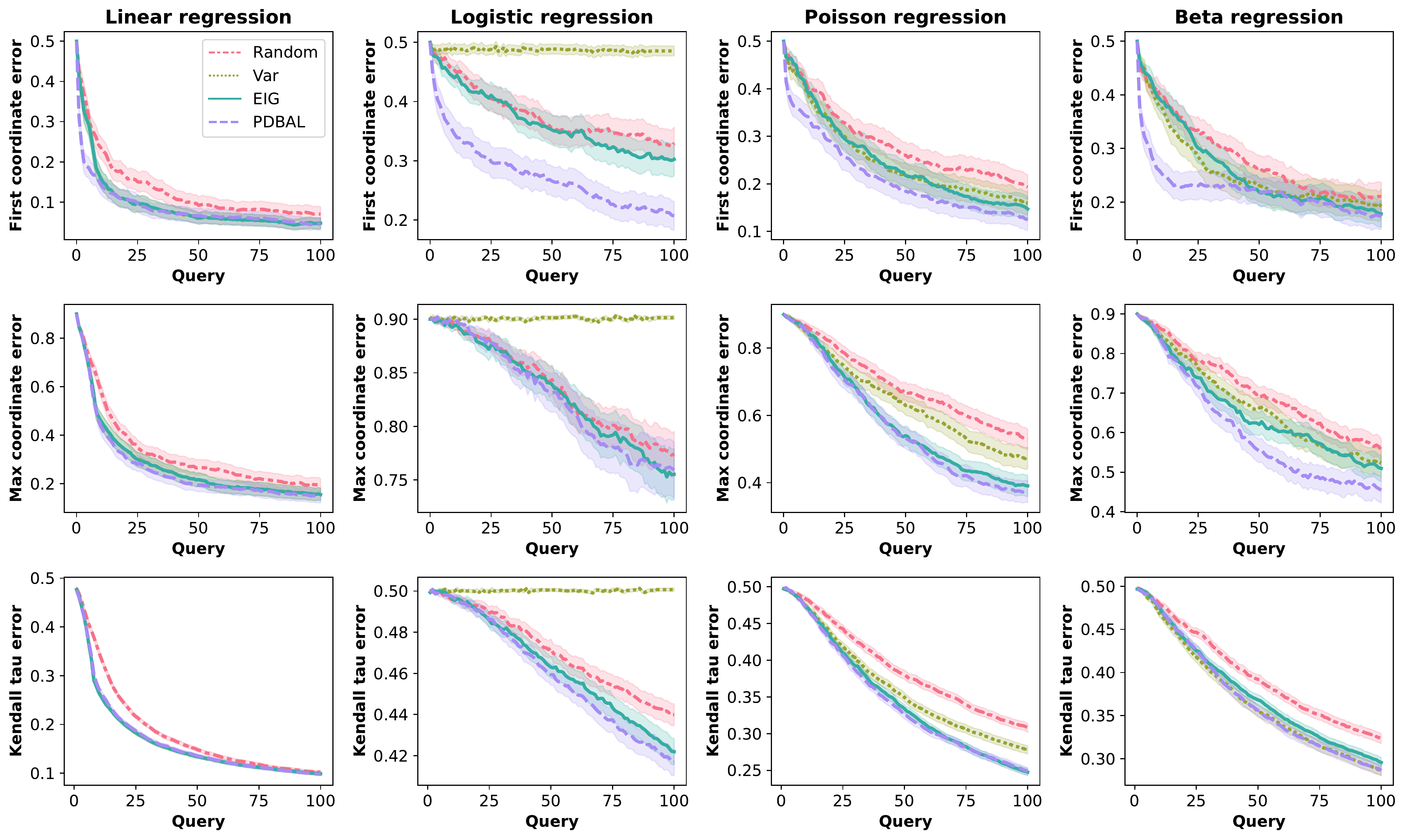}
\end{center}
\caption{Synthetic regression simulations. Columns correspond to the different probabilistic models, and rows correspond to the different objectives.}
\label{fig:simul_main}
\end{figure}

\paragraph{Results.} \pref{fig:simul_main} shows the results of our simulation study. In all settings, we see that \pdbal never does any worse than the baselines, and it does significantly better in some. We see that for the more focused objectives like the first coordinate identification, \pdbal has a much stronger separation from the untargeted baselines. We also see the influence of probabilistic model on performance, with the least gains coming from linear regression and the largest gains coming from logistic and Beta regression. This may be due to the fact that in our bounded linear regression setting, extreme values are unlikely to occur and so \pdbal is unable to make large changes to the posterior, whereas in Beta regression, values close to 0 and 1 are possible and allow for large changes to the posterior.

\subsection{Cancer drug screen experiments}

The Genomics of Drug Sensitivity in Cancer (GDSC) database~\citep{garnett2012systematic, yang2012genomics} is a large public database of cancer cell line experiments across a range of anticancer drug agents. Each drug is tested at a range of seven doses, allowing the full dose-response curve to be estimated. The reported outcome for each experiment is cell viability, defined as the proportion of cells alive after treatment. Large scale screens like GDSC are expensive and time-consuming; reducing the number of experiments required to accurately estimate responses and discover effective drugs could substantially expand the scale and speed of possible experiments.

We study the potential of \pdbal to adaptively screen anti-cancer drugs in a retrospective experiment on GDSC. At each step, the algorithm is allowed to conduct a single trial of a drug tested against a cancer cell line. We consider coarse- and fine-grained settings for drug selection. In the coarse setting, the algorithm selects the drug and cell line then observes the responses at each of the seven doses; in the fine-grain setting, the algorithm additionally specifies a dose. We used a subset of $100$ drugs and $20$ cell lines for a total of $n=2K$ possible coarse-grained experiments and $n=14K$ possible fine-grained experiments. Performance is evaluated as the error over each (cell line, drug, dose) when compared to what the underlying probabilistic model would learn from the full data.

All strategies use the same common Bayesian factor model,
\begin{equation}
\begin{aligned}\label{eq:gdsc-model}
y_{i{j_d}} &\sim \textrm{Normal}(\mu_{i{j_d}}, \sigma^2)
\\ \mu_{i{j_d}} &= a + b_i + c_{j_d} + \boldsymbol{v}_{i}^\top \boldsymbol{w}_{j_d},
\end{aligned}
\end{equation}
where  $i$ indexes the cell lines; $j_d$ indexes drug $j$ at dose $d$; $y_{i{j_d}}$ is the viability projected to the real line with a logistic transformation; $a$, $b_i$ and $c_{j_d}$ are scalar intercepts; and $\boldsymbol{w}_{i}$ and $\boldsymbol{v}_{j_d}$ are $q$-dimensional embeddings governing the interaction between cancer cell lines and drugs. The model is completed with hierarchical shrinking and smoothness-inducing priors. Similar factor models have previously been employed for modeling cancer drug screens~\citep{tansey2022bayesian}. \pref{app:gdsc-details} provides additional details about the model specification, experimental setup, and data pre-preprocessing.

To implement \pdbal, we use the mean-squared error distance of viability in probability space
\begin{align*}
    d_\text{v-mse}(\theta, \theta') =  \frac{1}{M} \sum_{i, j, d} (\text{sigmoid}(\mu_{i{j_d}}) - \text{sigmoid}(\mu_{i{j_d}}'))^2,
\end{align*}
where $\mu$ is the predictive mean in \cref{eq:gdsc-model} and $M$ is the total number of cell lines $\times$ drugs $\times$ doses.

\paragraph{Results.} \pref{fig:results-gdsc} shows the results of our experiments using the same baselines considered in~\Cref{subsec:empirical:synth}. The y-axis represents the target error $d_\text{v-mse}(\bar{\mu}_t, \bar{\mu}_*)$, where $\bar{\mu}_t=\mathbb{E}_{\theta \sim \pi_n}[\mu]$ is the posterior mean of the partial model fitted after observing a fraction $t$ of the data, and $\bar{\mu}_*=\mathbb{E}_{\theta^*\sim \pi_\text{full}}[\mu]$ the true predictive mean obtained from fitting the model with the full data.

In both experiments, \pdbal outperforms the other selection strategies. The fine-grained setting has more degrees of freedom and thus shows a larger performance gain for \pdbal. We did not evaluate EIG in the fine-grained experiment because the numerical integration required to approximate the objective was computationally prohibitive.
Additional details on the experiment results are in~\pref{app:gdsc-details}.

\begin{figure}[htb]
\begin{subfigure}[t]{0.5\columnwidth}
\centering
\includegraphics[width=\textwidth]{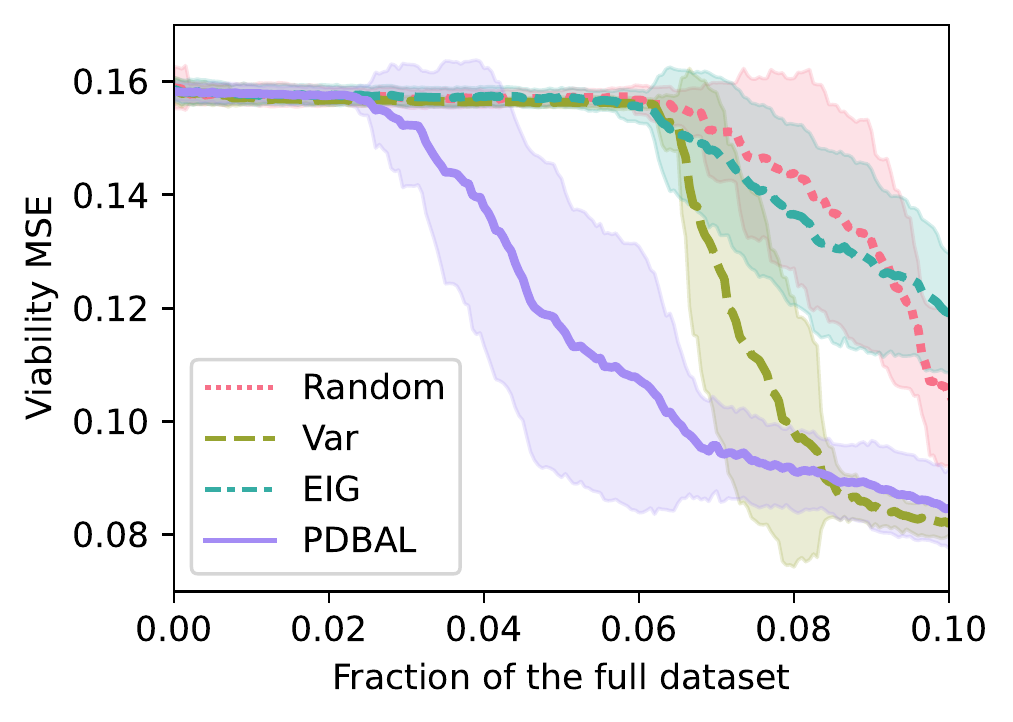}
\end{subfigure}\hfill
\begin{subfigure}[t]{0.5\columnwidth}
\centering
\includegraphics[width=\textwidth]{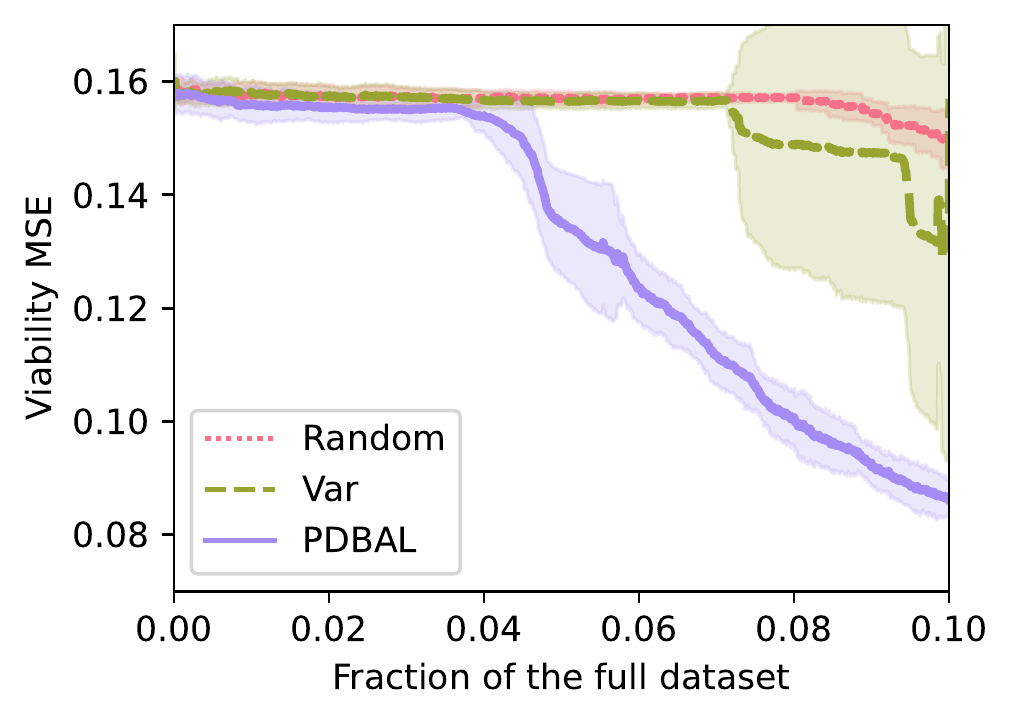}
\end{subfigure}
\caption{Results on GDSC. \emph{Left}: Coarse-grained experiment; each observation contains the full dose-response curve observations. \emph{Right}: Fine-grained experiment; each observation consists of a single dose selected by the active learning strategy.}
\label{fig:results-gdsc}
\end{figure}

To investigate the potential for early drug discovery, we conducted an additional evaluation comparing the ability to identify drugs that have targeted effects on specific cell lines. To measure selectivity, we follow the procedure outlined by \citet{tansey2022dose}, which we briefly summarize here for completeness. A drug is \emph{selective} if it has a dose at which it is safe on the majority of the cell lines but is highly toxic on at least one cell line. Here, safety and toxicity are defined as viability above $0.8$ and below $0.5$, respectively. The model trained on the whole dataset identified $9$ drugs from the pool of $100$. For each active learning method we used the corresponding model's posterior probability that a drug is selective to provide a ranking of drugs after observing 5\% and 10\% of the data.

\begin{table}[h]
\caption{Comparison of AUC for selective drug identification.} \label{sample-table}
\begin{center}
\begin{tabular}{l|cc}
 & {AUC at 5\%} & {AUC at 10\%} \\
\hline \\
\random  & 0.5 & 0.60  \\
\varsamp & 0.52 & 0.62  \\
\sceig   & 0.50 & 0.67  \\
\pdbal   & \textbf{0.60} & \textbf{0.71} 
\end{tabular}
\end{center}
\label{tab:selectivity}
\end{table}

\pref{tab:selectivity} presents the results for the coarse-grained drug selection experiments, showing the area under the curve (AUC) of the receiver operating characteristic (ROC) curve. The results show that \pdbal can more accurately discover selective drugs, notably even after only having selected 5\% of the experiments.

\section{Discussion}
\label{sec:discussion}

We introduced \pdbal, a targeted active learning algorithm compatible with any probabilistic model and any risk-aligned distance. We proved theoretical bounds on the query complexity of \pdbal, showing that in certain cases it is nearly optimal. In our simulation study, we showed the benefits of \pdbal in a range of settings.
On a real data example, \pdbal quickly converged to an accurate model with half the data required by other adaptive methods.

\paragraph{Limitations.} On the theoretical side, there are some concrete ways in which our results can be improved. The lower bounds we prove are limited to settings with small entropy, which limits the class of probabilistic models on which we can prove that we achieve near-optimality. Further, both our upper bounds and lower bounds rely on \pref{assum:homog-entropy}, which requires the conditional entropy to not depend on the parameter. Our simulation studies show \pdbal performs well even when these assumptions do not hold, suggesting that these limitations may be due to our analysis techniques rather than \pdbal itself. Finally, our general setup is within the well-specified Bayesian regime, which limits the applicability of this work to settings where we are comfortable assuming that our prior and likelihood are sufficiently accurate models of the real world.

\paragraph{Societal impact.} Any machine learning method carries risks of misapplications. We envision our work as being utilized by scientists to reduce the number of experiments needed to make scientific discoveries. This method, however, is agnostic to what those discoveries may be, which allows for the possibility of using this (and related) methodology to quickly uncover discoveries with negative societal consequences. We believe the immediate use cases in science outweigh these risks.

\subsubsection*{Acknowledgements}

WT is supported by grants from the Tow Center for Developing Oncology and the Break Through Cancer consortium.

\bibliography{refs}

\clearpage
\appendix
\onecolumn
\section{Motivating example: Linear regression}
\label{app:motive-example}

Consider a well-specified high-dimensional linear regression setting in which the covariates are drawn from the uniform distribution over the $d$-dimensional sphere of radius 1. For a data point $x^{(i)} \in \RR^d$, suppose the corresponding label is given by
\[ y_i \ = \ \langle \beta^\star, x^{(i)} \rangle + \epsilon_i \]
where $\epsilon_i \sim \Ncal(0,1)$ is independent noise and $\beta^\star$ is some ground-truth vector satisfying $\|\beta^\star\|_2 = 1$. Suppose further, that among the covariates there is a subset of $k \ll d$ (e.g., $k = O(\log d)$) covariates of interest, and the remaining coordinates are nuisance variables.

If we try to estimate the entire vector $\beta^\star$, then both active and passive learning will require $\Omega(d)$ queries to find some $\hat{\beta}$ such that $\| \hat{\beta} - \beta^\star \|^2$ is smaller than some constant. However, if we only care about estimating the coordinates of $\beta^\star$ that correspond to the covariates of interest, then this can be done using only $O(k)$ queries, given access to enough unlabeled data.

To see how, let $S \subset \{1, \ldots, n \}$ denote the coordinates of interest. For a vector $x \in \RR^d$, let $x_S \in \RR^k$ denote the $x$ restricted to $S$. Given a pool of unlabeled data, we will restrict our queries to the points $x$ such that $\| x_S \|_2^2$ is sufficiently large. Note that for any cutoff $\alpha \in (0,1)$ and integer $m \geq 1$, there is size $N \geq 1$ such that a random draw of $N$ points will have a subset of size $m$ whose elements satisfy $\| x_S \|_2^2 \geq 1-\alpha$.

Given $n \geq k$ queries that satisfy $X_S$ is full-rank, our estimator will be the least squares estimator defined over the coordinates $S$:
\[ \hat{\beta} \ = \ (X_S^T X_S)^{-1} X_S^T y. \]
Let $\tilde{y}_i = y_i - \langle \beta^\star_{S^c}, x^{(i)}_{S^c} \rangle$ denote the (unobserved) adjusted response values. In vector form, $\tilde{y} = y - X_{S^c} \beta^\star_{S^c}$. Then we can decompose the difference between $\hat{\beta}$ and $\beta^\star$ as
\begin{align*}
 \| \hat{\beta} - \beta^\star \| \ &= \ \| (X_S^T X_S)^{-1} X_S^T y -  \beta^\star \| \\
 \ &\leq \  \|(X_S^T X_S)^{-1} X_S^T \tilde{y} -  \beta^\star \| +  \| (X_S^T X_S)^{-1} X_S^T (y - \tilde{y}) \| \\
 \ &= \ \| (X_S^T X_S)^{-1} X_S^T \tilde{y} -  \beta^\star \| + \| (X_S^T X_S)^{-1} X_S^T X_{S^c} \beta^\star_{S^c} \| \\
 \ &= \  \| (X_S^T X_S)^{-1} X_S^T \epsilon \| + \| (X_S^T X_S)^{-1} X_S^T X_{S^c} \beta^\star_{S^c}\|.
\end{align*}
Here, the last line follows from the well-specified nature of the problem. We can bound each of the terms in the last line separately.

To analyze the first term, observe that since $\epsilon \sim \Ncal(0, I_n)$, we have 
\[ (X_S^T X_S)^{-1} X_S^T \epsilon \ \sim \ \Ncal \left(0, (X_S^T X_S)^{-1} \right) . \]
Bounding the first term comes down to bounding the $\ell_2$-norm of a random normal vector. The following lemma provides such bounds.

\begin{lemma}
\label{lemma:gaussian-l2-bound}
Let $\Sigma \in \RR^{d \times d}$ denote a positive definite covariance matrix with eigenvalues $\lambda_1, \ldots, \lambda_d$. Then for $v \sim \Ncal(0, \Sigma)$,
\begin{align*}
\pr\left( \| v \|_2^2 \leq t \sum_{j=1}^d \lambda_j \right) \ &\leq \ \sqrt{te} \ \ \ \text{ if } t < 1\\
\pr\left(\| v \|_2^2 \geq t \sum_{j=1}^d \lambda_j \right) \ &\leq \ \sqrt{t} e^{-(t-1)/2} \ \ \ \text{ if } t > 1.
\end{align*}
\end{lemma}
\begin{proof}
Let $Q \Lambda Q^T = \Sigma$ denote the eigendecomposition of $\Sigma$. Then observe that $z = \Lambda^{1/2} Q^T v \sim \Ncal(0, I_d)$ and
\begin{align*}
\| v \|_2^2 = v^T Q \Lambda Q^T v = u^T \Lambda u = \sum_{j=1}^d \lambda_j u_j^2
\end{align*}
where $u := Q^T v \sim \Ncal(0, I)$ (since $Q$ is orthonormal). Thus, we need to establish
\begin{align*}
\pr\left(  \sum_{j=1}^d \lambda_j u_j^2 \leq t \sum_{j=1}^d \lambda_j \right) \ &\leq \ \sqrt{te} \ \ \ \text{ if } t < 1\\
\pr\left( \sum_{j=1}^d \lambda_j u_j^2 \geq t \sum_{j=1}^d \lambda_j \right) \ &\leq \ \sqrt{t} e^{-(t-1)/2} \ \ \ \text{ if } t > 1
\end{align*}
for $u_1, \ldots, u_d \sim \Ncal(0,1)$ i.i.d.

We will show both inequalities through the Cram\`{e}r-Chernoff method~\citep[Chapter 2.2]{boucheron2013concentration}, starting with the first inequality. Fix any $\alpha \geq 0$ and denote $\lambda = \sum_{j=1}^d \lambda_j$, then 
\begin{align*}
\pr\left( \sum_{j=1}^d \lambda_j u_j^2 \leq t \sum_{j=1}^d \lambda_j \right) \ &= \ \EE\left[ \ind\left[ t \sum_{j=1}^d \lambda_j - \sum_{j=1}^d \lambda_j u_j^2 \geq 0  \right]  \right] \\
\ &\leq \ \EE\left[ \exp\left(\alpha \left( t \lambda - \sum_{j=1}^d \lambda_j u_j^2 \right) \right)  \right] \\
\ &= \ \exp\left(\alpha t \lambda \right) \prod_{j=1}^d \EE\left[ \exp \left(- \alpha \lambda_j u_j^2 \right) \right] \\
\ &= \ \exp\left(\alpha t \lambda \right)  \prod_{j=1}^d  \left(1 + 2 \alpha \lambda_j \right)^{-1/2} \\
\ &\leq \ \exp\left(\alpha t \lambda \right) \left(1 + 2 \alpha \lambda \right)^{-1/2}
\end{align*}
where the second-to-last line follows from the form of the Chi-squared moment generating function. Taking $\alpha = \frac{1}{2\lambda}\left( \frac{1}{t} - 1 \right)$ completes the proof of the first inequality.

For the second inequality, we follow similar steps, observing that for $\alpha \leq \frac{1}{2\lambda} \leq \min_j \frac{1}{2 \lambda_j}$ we have
\begin{align*}
\pr\left( \sum_{j=1}^d \lambda_j u_j^2 \geq t \sum_{j=1}^d \lambda_j \right) \ &= \ \EE\left[ \ind\left[ \sum_{j=1}^d \lambda_j u_j^2 -  t \sum_{j=1}^d \lambda_j \geq 0  \right]  \right] \\
\ &\leq \ \EE\left[ \exp\left(\alpha \left( \sum_{j=1}^d \lambda_j u_j^2 - t \lambda  \right) \right)  \right] \\
\ &= \ \exp\left( - \alpha t\lambda \right) \prod_{j=1}^d \EE\left[ \exp \left( \alpha \lambda_j u_j^2 \right) \right] \\
\ &= \  \exp\left( - \alpha t\lambda \right) \prod_{j=1}^d \left(1 - 2 \alpha \lambda_j \right)^{-1/2} \\
\ &\leq \ \exp\left( - \alpha t\lambda \right) \left(1 - 2 \alpha \lambda \right)^{-1/2}.
\end{align*}
Plugging in $\alpha = \frac{1}{2\lambda}\left( 1 -  \frac{1}{t} \right)$ gives us the second inequality.
\end{proof}

As a consequence of Lemma~\ref{lemma:gaussian-l2-bound}, we have with probability at least $1-\delta$,
\[ \| (X_S^T X_S)^{-1} X_S^T \epsilon \|_2^2 \ \leq \ \frac{k}{\lambda_{\min}(X_S^T X_S)} \log(1/\delta) \]
for small enough $\delta$. If our active learner draws a large enough collection of data points satisfying $\| x_S \|_2^2 \geq 1-\alpha$, then with high probability there will be a subset of $n$ data points such that the minimum eigenvalue of $X_S^T X_S$ is at least $(1-\alpha)n/2$. 
Combined with the above, this gives us
\[ \| (X_S^T X_S)^{-1} X_S^T \epsilon \|_2^2 \ \leq \ \frac{2k}{(1-\alpha) n} \log(1/\delta). \]
On the other hand, every row in $X_{S^c}$ has squared $\ell_2$-norm bounded by $\alpha$ and $\| \beta^\star \|_2^2 \leq 1$. Thus, a crude bound gives us
\[  \| (X_S^T X_S)^{-1} X_S^T X_{S^c} \beta^\star_{S^c}\|_2^2  \ \leq \ \frac{n \alpha}{\lambda_{\min}(X_S^T X_S)} \ \leq \ \frac{2 \alpha}{1-\alpha}. \]
Thus, for any target error $\epsilon > 0$ and failure probability $\delta > 0$, we can choose $\alpha = O\left( \epsilon \right)$ and $n = O\left( \frac{k}{\epsilon} \log \frac{1}{\delta }\right)$ and terminate with an estimate satisfying $\| \hat{\beta} - \beta^\star \|^2 \leq \epsilon$ with probability at least $1-\delta$.

\section{Proofs from~\pref{sec:algorithm}}
\label{app:alg-proofs}

\subsection{Proof of~\pref{prop:gaussian-shortcut}}
From the form of the Gaussian likelihood, we have
\begin{align*}
 &\EE_{y \sim \Ncal(\mu_1, \sigma_1^2 I_d)} \left[ \Ncal(y; \mu_2, \sigma_2^2 I_d) \Ncal(y; \mu_3, \sigma_3^2 I_d) \right] \\
 &= \int_{\RR^d} \left(\frac{1}{2\pi \sigma_1^2} \cdot \frac{1}{2\pi \sigma_2^2} \cdot \frac{1}{2\pi \sigma_3^2}\right)^{d/2} \exp \left( -\frac{1}{2\sigma_1^2}\|y - \mu_1  \|^2  -\frac{1}{2\sigma_2^2}\|y - \mu_2 \|^2 - \frac{1}{2\sigma_3^2}\|y - \mu_3  \|^2 \right) \, d y.
\end{align*} 
The bias-variance decomposition of squared error implies that for any $c_1, \ldots, c_n \geq 0$ and $b, b_1, \ldots, b_n \in \RR^d$, we have
\[ \sum_{i=1}^n c_i \| b_i - b \|^2 \ = \ \sum_{i=1}^n c_i \| b_i - \bar{b} \|^2 + {c}_{\text{sum}} \| b - \overline{b}\|^2  \]
where ${c}_{\text{sum}} = \sum_i c_i$ and $\overline{b} = \frac{1}{{c}_{\text{sum}}} \sum_i c_i b_i$. Applying this identity to the exponential above with the notation $\tau = \sum_{i=1}^3 \frac{1}{\sigma_i^2}$ and $\bar{\mu} = \frac{1}{\tau} \sum_{i=1}^3 \frac{\mu_i}{\sigma_i^2}$, we have 
\begin{align*}
  & \EE_{y \sim \Ncal(\mu_1, \sigma_1^2)} \left[\Ncal(y; \mu_2, \sigma_2^2) \Ncal(y; \mu_3, \sigma_3^2) \right] \\
 &= \left(\frac{1}{(2\pi)^3 \sigma_1^2 \sigma_2^2 \sigma_3^2} \right)^{d/2} \int_{\RR^d} \exp \left( - \sum_{i=1}^3 \frac{1}{2\sigma_i^2}\|\bar{\mu} - \mu_i \|^2  - \frac{\tau}{2} \|y - \bar{\mu} \|^2 \right) \, d y \\
 &= \left(\frac{1}{(2\pi)^3 \sigma_1^2 \sigma_2^2 \sigma_3^2}\right)^{d/2} \exp\left( - \sum_{i=1}^3 \frac{1}{2\sigma_i^2}\|\bar{\mu} - \mu_i \|^2 \right) \int_{\RR^d} \exp \left( - \frac{\tau}{2} \|y - \bar{\mu} \|^2 \right) \, d y \\
&= \left(\frac{1}{(2\pi)^3 \sigma_1^2 \sigma_2^2 \sigma_3^2}\right)^{d/2} \exp\left( - \sum_{i=1}^3 \frac{1}{2\sigma_i^2}\|\bar{\mu} - \mu_i \|^2 \right) \cdot \left(  \frac{2\pi}{\tau} \right)^{d/2} \\
&= \left(\frac{1}{(2\pi)^2 (\sigma_1^2 \sigma_2^2 + \sigma_2^2 \sigma_3^2 + \sigma_1^2\sigma_3^2)}\right)^{d/2} \exp\left( - \sum_{i=1}^3 \frac{1}{2\sigma_i^2}\|\bar{\mu} - \mu_i \|^2 \right)
\end{align*} 
where the second-to-last line follows from the fact that the integral is exactly the normalizing constant of a spherical Gaussian with mean $\bar{\mu}$ and variance $1/\tau$ and the last line follows from expanding the definition of $\tau$.

Any discrete random variable $X$ taking values $x_i$ with probability $p_i$ for $i=1,\ldots, m$ satisfies the following:
\[ \sum_{i=1}^m p_i \| x_i - \EE[X] \|^2 = \EE \| X - \EE[X]\|^2 = \sum_{1 \leq i < j \leq m} p_i p_j \| x_i - x_j\|^2. \]
Applying this to the discrete random variable that takes value $\mu_i$ with probability $\frac{1}{\tau \sigma_i^2}$,
\begin{align*}
\sum_{i=1}^3 \frac{1}{2\sigma_i^2}\|\bar{\mu} - \mu_i \|^2 
&= \frac{1}{\sigma_1^2 \sigma_2^2 \tau^2} \|\mu_1 - \mu_2 \|^2 + \frac{1}{\sigma_2^2 \sigma_3^2 \tau^2} \| \mu_2 - \mu_3 \|^2 + \frac{1}{\sigma_1^2 \sigma_3^2 \tau^2} \|\mu_1 - \mu_3 \|^2  \\
&= \frac{\sigma_1^2 \sigma_2^2 \sigma_3^2}{(\sigma_1^2 \sigma_2^2 + \sigma_2^2 \sigma_3^2 + \sigma_1^2\sigma_3^2)^2} \left( \sigma_{3}^2 \| \mu_{1} - \mu_{2} \|^2 
+ \sigma_{1}^2 \| \mu_{2} - \mu_{3} \|^2 
+ \sigma_{2}^2 \| \mu_{1} - \mu_{3} \|^2   \right).
\end{align*} 
Putting it all together gives us the desired identity. \qed

\subsection{Other closed form solutions}

\begin{proposition}
Fix $p^{(1)}, p^{(2)}, p^{(3)} \in \Delta^{K}$. Then
\[ \EE_{y \sim p^{(1)} }\left[ p^{(2)}_y  p^{(3)}_y \right] \ = \ \sum_{y=1}^K p^{(1)}_y p^{(2)}_y p^{(3)}_y.  \]
\end{proposition}
\begin{proof}
This follows immediately by substitution.
\end{proof}

\begin{proposition}
Fix $\lambda_1, \lambda_2, \lambda_3 > 0$, and let $f_\lambda(y) = \lambda e^{-\lambda y}$ denote the density function of the exponential distribution with parameter $\lambda$. Then
\[ \EE_{y \sim f_{\lambda_1}}\left[ f_{\lambda_2}(y)  f_{\lambda_3}(y) \right] \ = \ \frac{\lambda_1 \lambda_2 \lambda_3}{\lambda_1 + \lambda_2 + \lambda_3}.  \]
\end{proposition}
\begin{proof}
Simple calculus gives us
\begin{align*}
\EE_{y \sim f_{\lambda_1}}\left[ f_{\lambda_2}(y)  f_{\lambda_3}(y) \right] 
\ &= \ \int_{0}^\infty \lambda_1 e^{-\lambda_1 y} \lambda_2 e^{-\lambda_2 y} \lambda_3 e^{-\lambda_3 y} \, d y \\
\ &= \ \lambda_1 \lambda_2 \lambda_3 \int_{0}^\infty e^{-(\lambda_1 + \lambda_2 + \lambda_3) y} \, d y \\
\ &= \ \frac{\lambda_1 \lambda_2 \lambda_3}{\lambda_1 + \lambda_2 + \lambda_3}. \qedhere
\end{align*}
\end{proof}

\begin{proposition}
Fix $p_1, p_2, p_3 \in [0,1]$, and let $f_p(y) = p(1-p)^k$ denote the mass function of the geometric distribution with parameter $p$. Then
\[ \EE_{y \sim f_{p_1}}\left[ f_{p_2}(y)  f_{p_3}(y) \right] \ = \ \frac{p_1 p_2 p_3}{ p_1 + p_2 + p_3 - p_1 p_2 - p_2p_3 - p_1p_3 + p_1 p_2 p_3}.  \]
\end{proposition}
\begin{proof}
Expanding out, we have
\begin{align*}
\EE_{y \sim f_{p_1}}\left[ f_{p_2}(y)  f_{p_3}(y) \right] 
\ &= \ \sum_{y=0}^\infty p_1 p_2 p_3 \left( (1-p_1)(1-p_2)(1-p_3) \right)^k \\
\ &= \  p_1 p_2 p_3 \sum_{y=0}^\infty \left( 1 - \left(p_1 + p_2 + p_3 - p_1 p_2 - p_2p_3 - p_1p_3 + p_1 p_2 p_3 \right) \right)^k \\
\ &= \ \frac{p_1 p_2 p_3}{ p_1 + p_2 + p_3 - p_1 p_2 - p_2p_3 - p_1p_3 + p_1 p_2 p_3}. \qedhere
\end{align*}
\end{proof}

\begin{proposition}
Fix $\alpha_1, \alpha_2, \alpha_3, \beta_1, \beta_2, \beta_3 > 0$ such that $\alpha_1 + \alpha_2 + \alpha_3 > 2$, and let $f_{\alpha, \beta}(y) = \frac{\beta^\alpha}{\Gamma(\alpha)} x^{\alpha-1} e^{-\beta x}$ denote the density function of the gamma distribution with parameters $\alpha, \beta$. Then
\[ \EE_{y \sim f_{\alpha_1, \beta_1} }\left[f_{\alpha_2, \beta_2}(y)  f_{\alpha_3, \beta_3}(y) \right] \ = \ \frac{\beta_1^{\alpha_1} \beta_2^{\alpha_2} \beta_3^{\alpha_3} \Gamma( \alpha_1 + \alpha_2 + \alpha_3 - 2) }{(\beta_1 + \beta_2 + \beta_3)^{(\alpha_1 + \alpha_2 + \alpha_3 - 2)} \Gamma(\alpha_1) \Gamma(\alpha_2) \Gamma(\alpha_3) }.  \]
\end{proposition}
\begin{proof}
Expanding out, we have
\begin{align*}
 \EE_{y \sim f_{\alpha_1, \beta_1} } \left[ f_{\alpha_2, \beta_2}(y)  f_{\alpha_3, \beta_3}(y) \right] 
 \ &= \ \frac{\beta_1^{\alpha_1} \beta_2^{\alpha_2} \beta_3^{\alpha_3} }{\Gamma(\alpha_1) \Gamma(\alpha_2) \Gamma(\alpha_3)}
 \int_0^\infty  \left( x^{\alpha_1-1} e^{-\beta_1 x} \right) \left( x^{\alpha_2-1} e^{-\beta_2 x} \right) \left( x^{\alpha_3-1} e^{-\beta_3 x} \right) \, dx \\ 
  \ &= \ \frac{\beta_1^{\alpha_1} \beta_2^{\alpha_2} \beta_3^{\alpha_3} }{\Gamma(\alpha_1) \Gamma(\alpha_2) \Gamma(\alpha_3)}
 \int_0^\infty  x^{\alpha_1 + \alpha_2 + \alpha_3 - 3} e^{-(\beta_1 + \beta_2 + \beta_3 )x} \, dx \\ 
 \ &= \ \frac{\beta_1^{\alpha_1} \beta_2^{\alpha_2} \beta_3^{\alpha_3} \Gamma( \alpha_1 + \alpha_2 + \alpha_3 - 2) }{(\beta_1 + \beta_2 + \beta_3)^{(\alpha_1 + \alpha_2 + \alpha_3 - 2)} \Gamma(\alpha_1) \Gamma(\alpha_2) \Gamma(\alpha_3) },
\end{align*}
where the last line follows from the fact that the integral is the normalizing constant of a gamma distribution with parameters $\alpha_1 + \alpha_2 + \alpha_3 - 2$ and $\beta_1 + \beta_2 + \beta_3 $.
\end{proof}

\begin{proposition}
Fix $p_1, p_2, p_3 > 0$ and $r \geq 1$, and let $f_{r, p}(k) = { {k+r -1} \choose k} (1-p)^r p^k $ denote the probability mass function of the negative binomial distribution with parameters $r, p$. Then
\[ \EE_{k \sim f_{r, p_1} }\left[f_{r, p_2}(k)  f_{r, p_3}(k) \right] \ = \ \frac{\prod_{i=1}^3 (1 - p_i)^r}{(1- p_1 p_2 p_3)^r} \sum_{k=0}^{r-1} { {r-1} \choose k} 2^{-2k} \frac{\Gamma(2k + r)}{\Gamma(r) \Gamma(k+1)^2} \left( \frac{4 p_1 p_2 p_3}{(1-p_1 p_2 p_3)^2}  \right)^k  .\]
\end{proposition}
\begin{proof}
Expanding out, we have 
\begin{align*}
\EE_{k \sim f_{r, p_1} }\left[f_{r, p_2}(k)  f_{r, p_3}(k) \right] &= \sum_{k=0}^\infty { {k+r -1} \choose k}^3 \prod_{i=1}^3 (1-p_i)^r p_i^k \\
&= \left( \prod_{i=1}^3 (1-p_i)^r \right) \sum_{k=0}^\infty { {k+r -1} \choose k}^3 (p_1 p_2 p_3)^k \\
&= \left( \prod_{i=1}^3 (1-p_i)^r \right) {}_{3} F_{2} \left(r, r, r; 1,1; p_1 p_2 p_3 \right)
\end{align*}
where $ {}_{3} F_{2}(a,b,c; e,f; x)$ is the generalized hypergeometric function. From the identity
\[ {}_{3} F_{2}(a,b,c; a-b + 1, a-c +1; z) = (1-z)^{-a} {}_{3} F_{2} \left(a - b - c + 1, \frac{a}{2}, \frac{a+1}{2}; a-b + 1, a-c +1;  - \frac{4 z}{(1-z)^2} \right)  \]
we have 
\begin{align*}
{}_{3} F_{2} \left(r, r, r; 1,1; z \right) &=  (1-z)^{-r} {}_{3} F_{2} \left(1 - r, \frac{r}{2}, \frac{r+1}{2}; 1, 1; - \frac{4 z}{(1-z)^2} \right).
\end{align*}
Observe that whenever $m$ is a non-negative integer, we have
\[ {}_{3} F_{2} \left(-m, b, c; 1, 1; z \right) = \sum_{k=0}^m (-1)^k {m \choose k} \frac{\Gamma(b+k) \Gamma(c + k)}{\Gamma(b) \Gamma(c) \Gamma(k+1)^2} z^k. \]
Putting it together, we have for any $z > 0$ and any integer $r \geq 1$,
\begin{align*}
{}_{3} F_{2} \left(r, r, r; 1,1; z \right) &=  (1-z)^{-r} \sum_{k=0}^{r-1} (-1)^k {r-1 \choose k} \frac{\Gamma\left( \frac{r}{2} + k \right) \Gamma\left( \frac{r+1}{2} + k \right)}{\Gamma\left( \frac{r}{2} \right) \Gamma\left( \frac{r+1}{2} \right)\Gamma(k+1)^2} \left( - \frac{4 z}{(1-z)^2} \right)^k \\
&=  (1-z)^{-r} \sum_{k=0}^{r-1} \left(\frac{4 z}{(1-z)^2} \right)^k { {r-1} \choose k} 2^{-2k} \frac{\Gamma(2k + r)}{\Gamma(r) \Gamma(k+1)^2}
\end{align*}
where the last line follows from the Legendre duplication formula: 
\[ \Gamma(x) \Gamma(x + 1/2) = 2^{1-2x} \sqrt{\pi} \Gamma(2x). \]
Substituting in $z = p_1 p_2 p_3 $ gives us the proposition statement.
\end{proof}

\section{Proofs from~\pref{sec:theory}}

\subsection{Proof of \pref{lem:splitting}}

Let $Z_t = \EE_{\theta \sim \pi}[P_\theta(y_{1:t}; x_{1:t})]$, so that we have $W_t = Z_t e^{\sum_{i=1}^t H(x_i)}$. Observe that $Z_t$ is exactly the normalizing constant arising in Bayes' rule:
\[ \pi_t(\theta) =  \frac{1}{Z_t} \pi(\theta) P_\theta(y_{1:t}; x_{1:t}). \]
Thus, for any $x_{t+1}, y_{t+1}$, we have
\begin{align*}
\frac{Z_{t+1}}{Z_t} &= \frac{1}{Z_t} \EE_{\theta \sim \pi}\left[ P_\theta(y_{1:t+1}; x_{1:t+1})  \right] \\
&=  \EE_{\theta \sim \pi}\left[ \frac{P_\theta(y_{1:t}; x_{1:t})}{Z_t}  P_\theta(y_{t+1}; x_{t+1}) \right] \\
&= \EE_{\theta \sim \pi_t}\left[ P_\theta(y_{t+1}; x_{t+1}) \right].
\end{align*}
Moreover, by Bayes's rule, we also have
\[ \pi_{t+1}(\theta) 
=  \frac{ \pi_t(\theta) P_\theta(y_{t+1}; x_{t+1})}{\EE_{\theta \sim \pi_t}\left[ P_\theta(y_{t+1}, x_{t+1}) \right]} 
= \frac{Z_{t}}{Z_{t+1}} \pi_t(\theta) P_\theta(y_{t+1}; x_{t+1}). \]
Putting it all together, 
\begin{align*}
W_{t+1}^2 \avg(\pi_{t+1}) &= W_{t+1}^2 \EE_{\theta, \theta' \sim \pi_{t+1}}\left[ d(\theta, \theta') \right] \\
&=  W_{t+1}^2 \EE_{\theta, \theta' \sim \pi_{t}}\left[ \frac{Z_t^2}{Z_{t+1}^2}  P_\theta(y_{t+1}; x_{t+1}) P_{\theta'}(y_{t+1}; x_{t+1}) d(\theta, \theta') \right] \\
&= W_{t}^2 e^{2H(x_{t+1})} \EE_{\theta, \theta' \sim \pi_{t}}\left[ \frac{Z_t^2}{Z_{t+1}^2}  P_\theta(y_{t+1}; x_{t+1}) P_{\theta'}(y_{t+1}; x_{t+1}) d(\theta, \theta') \right].
\end{align*}
Taking expectations over $y_{t+1}$ and applying the definition of splitting finishes the argument. \qed

\subsection{Proof of \pref{prop:gaussian-dimension}}

Let $(x_1,y_1), \ldots, (x_n,y_n)$ be given. For $\theta, \theta' \in \Theta$,
\begin{align*}
\left| \sum_{i=1}^n \log P_\theta(y_i;x_i) - \log P_{\theta'}(y_i;x_i) \right|
&= \frac{1}{2\sigma^2}\left| \sum_{i=1}^n (y_i - \theta(x_i))^2 - (y_i - \theta'(x_i))^2\right| \\
&\leq \frac{1}{2\sigma^2}  \sum_{i=1}^n \left| (y_i - \theta(x_i))^2 - (y_i - \theta'(x_i))^2\right| \\
&\leq \frac{1}{2\sigma^2} \sum_{i=1}^n \left| \theta(x_i) - \theta'(x_i) \right| \left( 2|y_i| +  \left| \theta(x_i) + \theta'(x_i) \right| \right) \\
&\leq   \frac{2B}{\sigma^2} \sum_{i=1}^n \left| \theta(x_i) - \theta'(x_i) \right|.
\end{align*}
Thus, $N_{ll}(\epsilon, \Theta, n) \leq N_{1}(\sigma^2\epsilon/2B, \Theta, n)$, where $N_{1}$ denotes uniform covering with respect to $\ell_1$ distance. By known bounds on the covering number in terms of the pseudo-dimension, e.g.~\citep[Theorem~18.4]{anthony1999neural}, we have 
\[ N_{1}(\epsilon, \Theta, n) \leq e (d+1) \left( \frac{4eBn}{\epsilon} \right)^d. \]
Thus,
\[ N_{ll}(\epsilon, \Theta, n) \leq  e (d+1) \left( \frac{8eB^2n}{\epsilon \sigma^2} \right)^d. \] 
The definition of log-likelihood dimension finishes the argument. \qed

\subsection{Proof of \pref{prop:gaussian-entropy}}

For simplicity, let $\mu \in \RR$ and $\sigma^2 > 0$. Let $\Ncal(\cdot ; \mu, \sigma^2)$ denote the density of a Gaussian with mean $\mu$ and variance $\sigma^2$. Let $H = \frac{1}{2} + \frac{1}{2} \log \left( 2 \pi \sigma^2 \right)$ denote the entropy of $\Ncal(\cdot ; \mu, \sigma^2)$.

Suppose $Y \sim \Ncal(\mu, \sigma^2)$ and $X = \log \frac{1}{P_{\theta}(Y; x)} - H$, then
\begin{align*}
\EE[e^{\lambda X}] 
\ &= \ \EE\left[ \exp \left( \frac{\lambda}{2} \log \left( 2\pi \sigma^2 \right) + \lambda\frac{(Y -\mu)^2}{2 \sigma^2} - \frac{\lambda}{2} \log \left( 2\pi \sigma^2 \right) - \frac{\lambda}{2} \right) \right] \\
&=  e^{-\lambda/2} \EE\left[ \exp \left( \lambda \frac{(Y - \mu)^2}{2 \sigma^2} \right) \right] \\
&=  \frac{e^{-\lambda/2}}{(1-\lambda)^{1/2}}
\end{align*}
for $\lambda < 1$. Here, the last line follows from the fact that $\frac{(Y - \mu)^2}{\sigma^2}$ is chi-squared with one degree of freedom, and the known form of the chi-squared moment generating function. Thus, for $\lambda \in (0,1)$, we have
\[ \log \EE[e^{\lambda X}] = \frac{1}{2}\log \frac{1}{1-\lambda} - \frac{\lambda}{2} \ \leq \ \frac{\lambda^2}{2(1 - \lambda)}, \]
where the inequality follows from the bound $\log(x) \leq x - 1$ for $x > 0$. Thus, Gaussian location models are entropy sub-Gamma with variance factor 1 and scale parameter 1. \qed

\subsection{Proof of~\pref{lem:X_t-contraction}}
For $t \geq 1$, define $\Delta_t = 1 - \frac{W_t^2 \avg(\pi_{t})}{W_{t-1}^2 \avg(\pi_{t-1})}$. Let $\Fcal_{t}$ denote the sigma-field of all outcomes up to an including time $t$. Then if we query point $x_t$ which $\rho$-splits $\pi_{t+1}$, the definition of splitting implies that
\[ \EE[\Delta_{t+1} \mid x_t, \Fcal_{t}] \ \geq \ \rho. \]
Let $S_t = \sum_{i=1}^t (\Delta_t - \rho)$. The above implies that $S_t$ is a submartingale. Moreover, if $P_\theta(y; x) \leq c_1$ uniformly for all $\theta, x, y$ and $e^{H(x)} \leq c_2$ for all $x$, then
\begin{align*}
0 \ \leq \ \frac{W_{t+1}^2 \avg(\pi_{t+1})}{W_t^2 \avg(\pi_{t+1})} \ = \ e^{2H(x_{t+1})} \frac{\EE_{\theta, \theta' \sim \pi}[d(\theta, \theta') \prod_{i=1}^{t+1} P_{\theta}(y_i;x_i) P_{\theta'}(y_i;x_i) ]}{\EE_{\theta, \theta' \sim \pi}[d(\theta, \theta') \prod_{i=1}^{t} P_{\theta}(y_i;x_i) P_{\theta'}(y_i;x_i)]} \ \leq c_1^2 c_2^2.
\end{align*}
This implies $|S_{t+1} - S_t| \leq c_1^2 c_2^2$, and thus the Azuma-Hoeffding inequality~\citep{azuma1967weighted} tells us that with probability at least $1-\delta$,
\[ \sum_{i=1}^t \Delta_t \ = \ t\rho + S_{t} \ \geq \ t\rho - c_1 c_2 \sqrt{2t \log \frac{1}{\delta}}. \qed \]

\subsection{Proof of~\pref{lem:lower-bound-W_t}}

To prove~\pref{lem:lower-bound-W_t}, we will need the following lower bound.
\begin{lemma}
\label{lem:voronoi-partition}
Fix $\omega_n = ((x_1, y_1), \ldots, (x_n, y_n)) \in (\Xcal \times \Ycal)^n$ and let $M$ be an $\epsilon$-covering of $\Theta|_{\omega_n}$ with respect to $d_{ll}(\cdot, \cdot)$. Let $\Theta_1, \ldots, \Theta_{|M|}$ be the induced Voronoi partition of $\Theta$ (breaking ties arbitrarily). Then for any $\Theta_i$ and any $\theta^\star \in \Theta_i$
\[ \EE_{\theta \sim \pi}\left[ \prod_{i=1}^n P_{\theta}(y_i \mid x_i) \right] \ \geq \ e^{-2\epsilon} \pi(\Theta_i) \prod_{i=1}^n P_{\theta^\star}(y_i \mid x_i). \]
\end{lemma}
\begin{proof}
Fix $\Theta_i$ and let $m_i \in M$ denote the Voronoi `center.' Then for any $\theta, \theta' \in \Theta_i$, we have
\[ d_{ll}(\theta, \theta') \ \leq \ d_{ll}(\theta, m_i) + d_{ll}(m_i, \theta') \ \leq \ 2 \epsilon,  \]
where we have used the fact that $M$ is an $\epsilon$-cover. Thus, we have
\[ \prod_{i=1}^n P_{\theta'}(y_i \mid x_i) \ \geq \ e^{-2\epsilon} P_{\theta}(y_i \mid x_i) .\]
Finally, because $\Theta_1, \ldots, \Theta_{|M|}$ are disjoint, we have
\begin{align*}
\EE_{\theta \sim \pi}\left[ \prod_{i=1}^n P_{\theta}(y_i \mid x_i) \right]
\ \geq \ \sum_{j=1}^n \pi(\Theta_j) \EE_{\theta \sim \pi}\left[ \prod_{i=1}^n P_{\theta}(y_i \mid x_i) \mid \theta \in \Theta_j  \right]
\ \geq \ e^{-2\epsilon} \pi(\Theta_i) P_{\theta^\star}(y_i \mid x_i). 
\end{align*}
\end{proof}

We will also require the following result which follows directly from well-known tail bounds for sub-Gamma random variables.
\begin{lemma}
\label{lem:sub-Gamma-concentration}
Suppose $\Theta$ is entropy sub-Gamma with variance factor $v > 0$ and scale parameter $c>0$. If $y_i \sim P_{\theta}(\cdot; x_i)$ for $i=1, \ldots, t$, then with probability at least $1 - \delta$,
\[ \sum_{i=1}^t \log \frac{1}{P_{\theta}(y_i; x_i)} \leq \sum_{i=1}^t H(x_i) + \sqrt{ 2 t v \log \frac{1}{\delta}} + c \log \frac{1}{\delta} .\]
\end{lemma}
\begin{proof}
Observe that for independent mean-zero sub-Gamma random variables $X_1, \ldots, X_n$ with variance factor $v >0$ and scale parameter $c > 0$, the random variable $Z = \sum_{i} X_i$ satisfies
\[ \log \EE\left[ e^{\lambda Z}  \right] = \log  \EE\left[ \prod_{i=1}^n e^{\lambda X_i}  \right] = \sum_{i=1}^n  \log  \EE\left[ e^{\lambda X_i}  \right] \leq  \frac{n v \lambda^2}{2(1 - c\lambda)}  \]
for $\lambda \in (0,1/c)$.
Thus, $Z$ is sub-Gamma with variance factor $nv$ and scale parameter $c$. The argument is finished via standard concentration results on sub-Gamma random variables, e.g. \citep[Chapter 2.4]{boucheron2013concentration}.
\end{proof}

With Lemmas~\ref{lem:voronoi-partition} and~\ref{lem:sub-Gamma-concentration} in hand, we can turn to proving~\pref{lem:lower-bound-W_t}.
\begin{proof}[Proof of~\pref{lem:lower-bound-W_t}]
Recall that we can write
\[  W_t \ = \ \exp\left(  \sum_{i=1}^t H(x_i) \right) \EE_{\theta \sim \pi} \left[ \prod_{i=1}^t P_{\theta}(y_i ; x_i) \right]. \]
Let $\omega_t = ((x_1, y_1), \ldots, (x_t, y_t))$ denote our data. By Lemma~\ref{lem:voronoi-partition}, we have
\[ \EE_{\theta \sim \pi}\left[ \prod_{i=1}^t P_{\theta}(y_i \mid x_i) \right] \ \geq \ e^{-2} \pi(\Theta_i) \prod_{j=1}^t P_{\theta_i}(y_j \mid x_j) \]
where $\Theta_1, \ldots, \Theta_M$ is the Voronoi partition induced by a minimal $1$-covering of $\Theta|_{\omega_n}$, $\Theta_i$ is any of these partition elements, and $\theta_i$ is any element in $\Theta_i$. Let $S$ denote the set of indices $i$ such that $\pi(\Theta_i) < \delta/(2M)$. Then we have
\[ \pr \left(  \exists i \in S \text{ s.t. } \theta^\star \in \Theta_i \right) \ = \ \sum_{i \in S} \pi(\Theta_i) \ \leq \ \frac{|S| \delta}{2 M} \ \leq \ \frac{\delta}{2}. \]
Thus, if $\Theta^\star$ is the element of the partition that $\theta^\star$ falls into, we have with probability at least $1-\delta/2$
\begin{align*}
\EE_{\theta \sim \pi}\left[ \prod_{i=1}^t P_{\theta}(y_i \mid x_i) \right] 
\ &\geq \ e^{-2} \pi(\Theta^\star) \prod_{j=1}^t P_{\theta^\star}(y_j \mid x_j) \\
\ &\geq \ \frac{1}{M} \exp\left( -2 - \log \frac{2}{\delta} \right) \prod_{j=1}^t P_{\theta^\star}(y_j \mid x_j)   \\
\ &\geq \ \exp\left( -2 - d \log (c t) - \log \frac{2}{\delta} \right) \prod_{j=1}^t P_{\theta^\star}(y_j \mid x_j),
\end{align*}
where the last line follows from the fact that $M \leq (c t)^d$.

Finally, observe that with probability at least $1- \delta/2$, Lemma~\ref{lem:sub-Gamma-concentration} implies
\[ \prod_{j=1}^t P_{\theta^\star}(y_j \mid x_j) 
\ = \ \exp\left( - \sum_{j=1}^t \log \frac{1}{ P_{\theta^\star}(y_j \mid x_j)} \right) 
\ \geq \ \exp\left( -\sum_{i=1}^t H(x_i) - \sqrt{2 t v \log \frac{2}{\delta}} - c' \log \frac{2}{\delta} \right). \]
A union bound finishes the argument.
\end{proof}

\subsection{Proof of~\pref{thm:dbal-upper-bound}}
Combining Lemma~\ref{lem:X_t-contraction} with a union bound, we have that with probability $1-\delta/2$
\[ W_{t}^2 \avg(\pi_t) \ \leq \ \exp\left(- t\rho + c_1 c_2 \sqrt{2t \log \frac{2t(t+1)}{\delta}} \right) \avg(\pi) \]
for all $t\geq 1$, simultaneously. Similarly, combining Lemma~\ref{lem:lower-bound-W_t} with a union bound gives us with probability at least $1-\delta$,
\[ W_t^2 \ \geq \ \exp\left( - d \log(c t) - 2 - \log \frac{4t(t+1)}{\delta}- \sqrt{ 2 t v \log \frac{4t(t+1)}{\delta}} -  c' \log \frac{4t(t+1)}{\delta} \right) \]
for all $t\geq 1$. Thus, with probability $1-\delta$, both of these occur simultaneously. Plugging in the value of $T$ from the theorem statement,
\begin{align*}
\avg(\pi_T) \ &\leq \ \avg(\pi) \exp\left(- T\rho + c_1 c_2 \sqrt{2T \log \frac{2T(T+1)}{\delta}} + d \log(c T) \right.\\
&\hspace{5em} \left. + 2 + \log \frac{4T(T+1)}{\delta} + \sqrt{2 T v \log \frac{4T(T+1)}{\delta}} +  c' \log \frac{4T(T+1)}{\delta}  \right) \\ &\leq \ \avg(\pi) \exp\left(- T\rho + 4c_1 c_2 \sqrt{T v \log \frac{4T(T+1)}{\delta}} + (d + c' + 1)\log \left( \frac{4c T(T+1)}{\delta} \right) \right)
\end{align*}
The above is less than $\epsilon$ when we have
\[ T \geq \max \left\{ 9 \left( \frac{12 c_1 c_2}{\rho} \right)^2 v \log\left( \left( \frac{12 c_1 c_2}{\rho} \right)^2 \cdot \frac{4}{\delta} \right),
\frac{27}{\rho}(d + c' +1) \log\left( \frac{3}{\rho}(d + c' +1) \cdot \frac{4c}{\delta} \right), \frac{3}{\rho} \log \frac{\avg(\pi)}{\epsilon} \right\}.\]
Here, we have made use of the fact that if $a \geq 1$, $b \geq e$ and $x \geq 9a \log(ab)$, then $x \geq a \log(bx (x+1)).$ \qed

\subsection{Proof of~\pref{lem:combine-split}}
Observe by the product measure assumption of $P_\theta(\cdot ; x_1, x_2)$, we have
\begin{align*}
&\EE_{\theta^\star \sim \pi} \EE_{y_1,y_2 \sim P_{\theta^\star}(x_1, x_2)} \EE_{\theta, \theta' \sim \pi}
\left[  P_\theta(y_1, y_2; x_1, x_2) P_{\theta'}(y_1, y_2; x_1, x_2) d(\theta, \theta') \right] \\
&= \EE_{\theta, \theta', \theta^\star \sim \pi} \left[ d(\theta, \theta') 
\EE_{y_1\sim P_{\theta^\star}(x_1)}\left[ P_{\theta}(y_1; x_1) P_{\theta'}(y_1; x_1) \right] 
\EE_{y_2\sim P_{\theta^\star}(x_2)}\left[ P_{\theta}(y_2; x_2) P_{\theta'}(y_2; x_2) \right]   \right] \\
&=: \EE_{\theta, \theta', \theta^\star \sim \pi} \left[ d(\theta, \theta') 
 \alpha_1(\theta, \theta', \theta^\star)  \alpha_2(\theta, \theta', \theta^\star)  \right].
\end{align*}
Let $U$ be the random variable that takes on value $\alpha_1(\theta, \theta', \theta^\star)$ and let $V$ denote the random variable that takes on value $\alpha_2(\theta, \theta', \theta^\star)$. Here, $\theta, \theta', \theta^\star$ occur with probability $\frac{\pi(\theta) \pi(\theta') \pi(\theta^\star) d(\theta, \theta')}{\avg(\pi)}$. Then it is not hard to see that $\EE[U] = (1- \rho_1)e^{-2h}$ and $\EE[V] = (1 - \rho_2)e^{-2h}$. Note that $U$ and $V$ lie in the interval $[0,1]$ almost surely. Let $A = 1-U$ and $B=1-V$.

Let us first consider the case where $h = 0$. Then we have
\[ \EE[AB] = 1 - \EE[U] - \EE[V] + \EE[UV] = 1 - (1-\rho_1) - (1-\rho_2) + \EE[UV] \ = \ \rho_1 + \rho_2 - 1 + \EE[UV]. \]
Observe that $AB \geq 0$ almost surely, and so we have
\[ \EE[UV] \ \geq \ 1 - \rho_1 - \rho_2. \]
Substituting in our definitions of $U$ and $V$ gives us the result.

Now consider the case where $0 \leq h \leq \frac{\rho_1 + \rho_2}{6}$. The same argument as before shows that
\[ \EE[UV] \geq  (2 - \rho_1 - \rho_2)e^{-2h} - 1 = (2 - \rho)e^{-2h} - 1, \]
where we have made the substitution $\rho = \rho_1 + \rho_2$. To prove the lemma, we will show that the above is greater than $(1- 2\rho)e^{-4h}$. This is equivalent to showing
\[  (2 - \rho)e^{2h} - e^{4h} - 1 + 2 \rho \geq 0. \]
The left-hand side is decreasing for $h \geq 0$. Moreover, we also have the inequality $h \leq \frac{\rho}{6} \leq \frac{1}{2} \log(1 + \frac{\rho}{2})$. Thus,
\begin{align*}
(2 - \rho)e^{2h} - e^{4h} - 1 + 2 \rho  \geq (2 - \rho)\left(1 + \frac{\rho}{2} \right) - \left(1 + \frac{\rho}{2} \right)^2 - 1 + 2 \rho = \rho - \frac{\rho^2}{2} \geq 0.  \qed %
\end{align*}

\subsection{Proof of~\pref{thm:deterministic-lower-bound}}

Let $\pi$ be a prior distribution as in the theorem statement. Suppose we draw less than $1/2\tau$ unlabeled examples, then with probability at least $(1-\tau)^{1/2\tau} \geq 1/2$ none of these $\rho$-split $\pi$. Let us condition on this event. By induction on Lemma~\ref{lem:combine-split}, we have that any collection of $n \leq 1/\rho$ of these points does not $n \rho$-split $\pi$. 

Suppose that we query $n$ of these points (say $x_1, \ldots, x_n)$, and receive responses $y_1, \ldots, y_n$. Let $\pi_n$ denote this posterior. By \pref{lem:splitting}, we have
\begin{align*}
\EE_{y_{1:n}} \left[ Z_n^2 \avg(\pi_n) \right] 
& \geq (1 - n \rho) \avg(\pi),
\end{align*}
where $Z_n = \EE_{\theta \sim \pi}\left[ P_{\theta}(y_{1:n}; x_{1:n}) \right] \leq 1$. Thus,
\[ \EE_{y_{1:n}} \left[ \avg(\pi_n) \right] \ \geq \  (1 - n \rho) \avg(\pi). \]
For a random variable $U$ satisfying $U \leq c$ almost surely, the reverse Markov inequality gives us
\[ \pr\left( U > \alpha \right) \geq \frac{\EE[U] - \alpha}{c - \alpha}\]
for any $\alpha \leq \EE[X]$. Applying this to the random variable $\frac{\avg(\pi_n)}{\avg(\pi)}$ and assuming $n \leq \frac{1}{2\rho}$, we have that $\avg(\pi_n) \geq \epsilon$ with probability at least $1/3$. Putting it all together gives us the theorem statement. \qed

\subsection{Proof of~\pref{thm:lowent-lower-bound}}

We will need the following lemma.
\begin{lemma}
\label{lem:all-split-lowent}
Let $n \leq \sqrt{1/\rho}$, and let $h \in \RR$ satisfy $0 \leq h \leq $. Suppose $x_1 , \ldots, x_n$ all satisfy $H(x_i) = h \leq \rho/6n$ and have splitting index $\leq \rho$. Then the combined query $x_{1:n}$ has splitting index less than $n^2 \rho$.
\end{lemma}
\begin{proof}
We will show the claim for $n$ a power of 2. Extending to other integers is straightforward.

The proof is by induction. Where we first observe that any subsequence $i_1, i_2, \ldots, i_k \in \{1,\ldots,n\}$ satisfies that 
\[ H(x_{i_1}, \ldots, x_{i_k}) = \sum_{j=1}^k H(x_{i_j}) = k h \leq \frac{\rho k}{6n} \leq \frac{\rho}{6}. \]
Now for $n=1$, then the claim trivially holds. For $n \geq 2$, observe that by our inductive hypothesis, we have $x_{1:n/2}$ and $x{n/2+1:n}$ each have splitting index less than $\frac{n^2 \rho}{4}$. Applying \pref{lem:combine-split}, completes the argument.
\end{proof}

Turning to the proof of~\pref{thm:lowent-lower-bound}, let $\pi$ be a prior distribution as in the theorem statement. Suppose we draw less than $1/2\tau$ unlabeled examples, then with probability at least $(1-\tau)^{1/2\tau} \geq 1/2$ none of these $\rho$-split $\pi$. Let us condition on this event.

Now let $n \leq \frac{1}{2\sqrt{\rho}}$. Using the fact that $n < \sqrt{1/\rho}$ and $h < \rho^{3/2}/6$, we can apply \pref{lem:all-split-lowent}, to see that any collection of $n$ of these points does not $n^2 \rho$-split $\pi$. \pref{lem:splitting} then implies that
\begin{align*}
\EE_{y_1, \ldots, y_n} \left[ W_n^2 \avg(\pi_n) \right] 
& \geq (1 - n^2 \rho) \avg(\pi).
\end{align*}
Recall $W_n = e^{\sum_{i=1}^n H(x_i)} \EE_{\theta \sim \pi} \left[ \prod_{i=1}^n P_{\theta}(y_i; x_i) \right]$. By our assumptions that $H(x) \leq \rho^{3/2}/6$, $P_{\theta}(y; x) \leq 1$, $n \leq \sqrt{1/\rho}$ and $\rho \leq 1/4$, we have $W_n \leq 3/2$ almost surely.
Thus,
\[ \EE_{y_{1:n}} \left[ \avg(\pi_n) \right] \ \geq \  \frac{2}{3} (1 - n^2 \rho) \avg(\pi). \]
For a random variable $U$ satisfying $U \leq c$ almost surely, the reverse Markov inequality gives us
\[ \pr\left( U > \alpha \right) \geq \frac{\EE[U] - \alpha}{c - \alpha}\]
for any $\alpha \leq \EE[X]$. Applying this to the random variable $U = \frac{\avg(\pi_n)}{\avg(\pi)}$ and threshold $\alpha = \frac{\epsilon}{\avg(\pi)}$, we have 
\begin{align*}
 \pr\left( \avg(\pi_n) >  \epsilon \right) &\geq \frac{\frac{2}{3} (1 - n^2 \rho) \avg(\pi) - \epsilon}{\avg(\pi) - \epsilon} \\
 &\geq \frac{ \frac{8\epsilon}{3}(1 - n^2 \rho) - \epsilon }{4\epsilon - \epsilon} > \frac{1}{3}
\end{align*}
where we have used the fact that $n \leq \frac{1}{2\sqrt{\rho}}$ and $f(x) = \frac{cx - \epsilon}{x - \epsilon}$ is increasing in $x$ when $c \in (0,1)$. Thus, $\avg(\pi_n) \geq \epsilon$ with probability at least $1/3$, finishing the argument. \qed

\section{Additional details on the simulation study}
\label{app:more-simul}

\begin{figure}
\includegraphics[width=\textwidth]{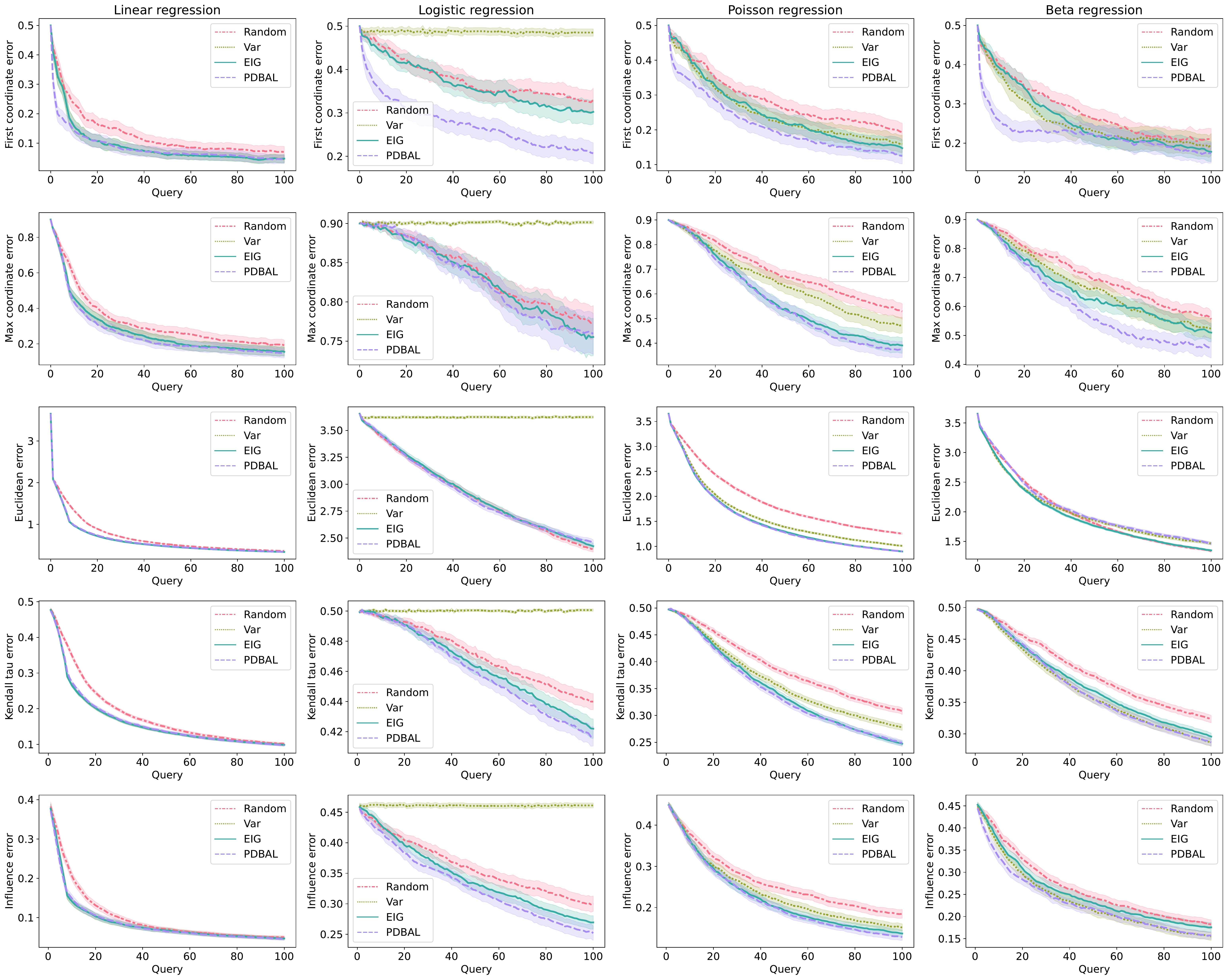}
\vspace{-1.5em}
\caption{Full set of synthetic regression simulations.}
\label{fig:simul_all}
\end{figure}

Here we further expand on the details of our regression setup. We consider the following models.
\begin{itemize}
	\item Linear regression with homoscedastic Gaussian noise: 
	\[ P_\theta(y ; x) = \Ncal(y \mid \langle x, \theta \rangle, \sigma^2), \]
	where $\sigma > 0$ is some known standard deviation. In our experiments it is set to $1/4$.
	\item Logistic regression:
	\[ P_\theta(y ; x) = \text{Bernoulli} \left(y \mid \mu \right), \]
	where $\mu = \frac{1}{1 + e^{-\langle x, \theta \rangle}}$.
	\item Poisson regression:
	\[ P_\theta(y ; x) = \text{Poisson} \left( y \mid  e^{\langle x, \theta \rangle} \right). \]
	\item Beta regression using the well-known mean parameterization~\citep{ferrari2004beta}:
	\[ P_\theta(y ; x) = \text{Beta} \left( y \mid  \phi \mu, \phi(1- \mu) \right), \]
	where $\mu = \frac{1}{1 + e^{-\langle x, \theta \rangle}}$ and $\phi>0$ is a fixed and known constant.
\end{itemize}

We consider five objectives and their corresponding distances over parameters.
\begin{itemize}
	\item What is the sign of the first coordinate?
	\[ d_{\text{first}}( \theta, \theta' ) = \ind[ \sign(\theta_1) \neq \sign(\theta'_1)]. \]
	\item Which coordinate has largest absolute magnitude?
	\[ d_{\text{max}}( \theta, \theta' ) = \ind[ \argmax_{i} |\theta_i| \neq \argmax_{i} |\theta'_i| ]. \]
	\item Can we identify the parameter?
	\[ d_{\text{euclidean}}( \theta, \theta' ) = \| \theta - \theta \|_2. \]
	\item What is the order of the magnitudes of the coordinates?
	\[ d_{\text{kendall}}( \theta, \theta' ) =  \frac{1}{2} \left( 1 - \tau(|\theta|, |\theta'|)  \right) \]
	where $\tau(|\theta|, |\theta'|)$ is the Kendall's $\tau$ correlation between of the pairs $(|\theta_1|, |\theta_1'|), \ldots, (|\theta_d|, |\theta_d'|)$.
	\item What is the influence of the first $d/2$ coordinates on the predicted sign?
	\[ d_{\text{influence}}( \theta, \theta' ) =  \pr_{x} \left(  \sign \left(\langle x_{1:d/2}, \theta_{1:d/2}  \rangle\right) \neq  \sign \left(\langle x_{1:d/2}, \theta'_{1:d/2}  \rangle\right)\right), \]
	where $x_{1:d/2}$ denotes $x$ restricted to its first $d/2$ coordinates.
\end{itemize}

For each query, we first sampled 2K fresh data points from the distribution and then chose the next query from this pool of points. For methods requiring posterior samples, we collected 300 MCMC samples from the NUTS algorithm, using 2 parallel chains, a burnin of 750 steps, and a thinning factor of 5. We evaluated the model by drawing the same number of MCMC samples with the same parameters.

\pref{fig:simul_all} depicts the results of all of our simulations. One notable observation is that \varsamp does very poorly on logistic regression tasks. The reason for this is that in our data generating process, it is possible to sample data points whose covariates are all zero. Such data points maximize posterior predictive variance but provide no information on the underlying regression coefficients.

\section{Additional details on the drug discovery experiment}
\label{app:gdsc-details}

\paragraph{Data preprocessing} We downloaded the publicly available \texttt{GDSC2-raw-data} dataset from \url{https://www.cancerrxgene.org/downloads/bulk_download} and pre-processed it using the R package \texttt{gdscIC50} \citep{lightfoot2016gdscic50}. This preprocessing step transforms the raw cell counts of each experiment to cell viability (fraction of surviving cells) adjusting for low and high dimethyl sulfoxide (DMSO) controls. We then obtain the subset of experiments corresponding to the top $M=20$ cell lines and $L=100$ drugs that appear the most times in the dataset, at all $7$ concentrations. Then, we transform the viability using the logistic transform by setting $y_{ij_d}=\text{logit}(\text{clip}(\text{viability}_{ij_d}, 0.005, 0.995))$. Since there are multiple observations for each cell line/drug/dose triplet, we aggregate them by averaging, yielding our final dataset.

\paragraph{Statistical model} We fit a Bayesian factor model to the full dataset and use the predictions $\mu_{ij_d}$ as the reference for computing the progress of active learning. As more experiments are revealed, all active learning strategies converge to the same solution given by the full data model fit. The model has the form
\begin{equation*}
\begin{aligned}
y_{i{j_d}} &\sim \textrm{Normal}(\mu_{i{j_d}}, \sigma^2)
\\ \mu_{i{j_d}} &= a + b_i + c_{j_d} + \boldsymbol{w}_{i}^\top \boldsymbol{v}_{j_d}.
\end{aligned}
\end{equation*}
The model is completed with standard Horseshoe priors \citep{carvalho2009handling} for regularization and an auto-regressive prior \citep{besag1974spatial} to encourage smoothness along the dose-response curves
\begin{equation*}
\begin{aligned}
b_i &\sim \textrm{Normal}( 0, \lambda_{b_i}^2) &\\
w_{i,r} &\sim \textrm{Normal}(0, \lambda_{w_{i,r}}^2) &\\
(c_{j_1},\hdots,c_{j_7}) &\sim AR(\eta) \times \Pi_{d=1}^7\textrm{Normal}(0, \lambda_{c_{j_d}}^2)  \\
({w}_{j_{1,r}},\hdots,{w}_{j_{7,r}})& \sim AR(\eta) \times \Pi_{d=1}^7\textrm{Normal}(0, \lambda_{{w}_{j_d,r}}^2)  \\
\lambda_{b_i} & \sim C^+(0,1) \\
\lambda_{c_{j_d}} & \sim C^+(0,1) \\
\lambda_{w_{i,r}} & \sim C^+(0,1) \\
\lambda_{{v}_{j_d,r}} & \sim  C^+(0,1) \\
(1/\sigma^2) & \sim \text{Exp}(1)
\end{aligned}
\end{equation*}
where $\boldsymbol{x} \sim AR(\eta)$ for a vector $\boldsymbol{x}\in\mathbb{R}^d$ means $p(x \mid \eta) \propto \exp(-(\eta/2)\sum_{s=2}^d (x_s-x_{s-1})^2)$.  We set $\eta=0.1$ to add smoothness to the prior along the dose-response curve, but do not fine-tune this parameter. The embedding dimension is set to $q=4$, which we find suffices to provide a good fit to the data. \pref{fig:gdsc-model-fit} shows the fit to the data and examples of dose-response curves. Despite the low dimensionality of the embeddings, the model is able to capture over 75\% of the variation in viability. Bayesian inference is conducted using a simple Gibbs sampler which can be derived analytically in closed conjugate form. To do so, we expand the half-Cauchy prior into a scale mixture to get updates that only involve normal and inverse-gamma distributions. Since the model is Gaussian, the \pdbal scores are evaluated using the formula in \pref{prop:gaussian-shortcut}.
\begin{figure}
    \centering
    \begin{subfigure}[b]{0.6\textwidth}
    \centering
    \includegraphics[width=0.9\textwidth]{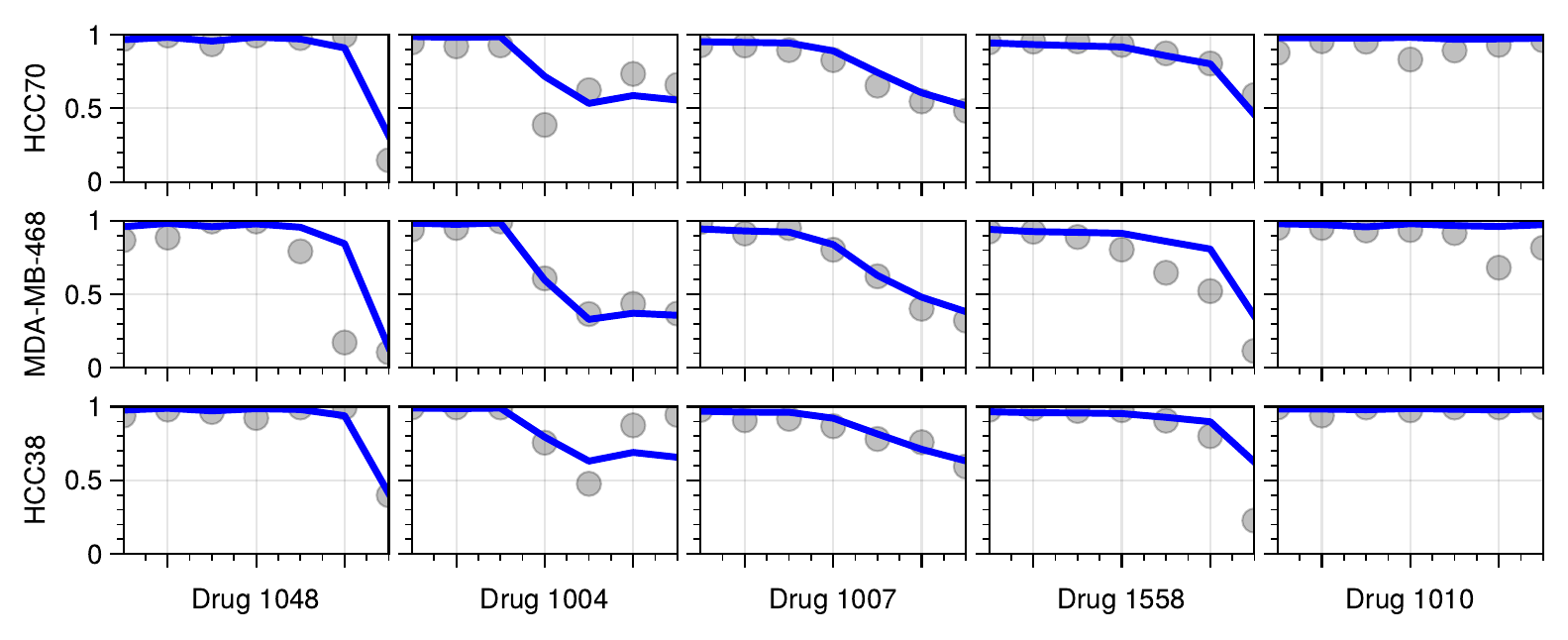}
    \end{subfigure}\hspace{-2cm}~
    \begin{subfigure}[b]{0.4\textwidth}
    \centering
    \includegraphics[width=0.63\textwidth]{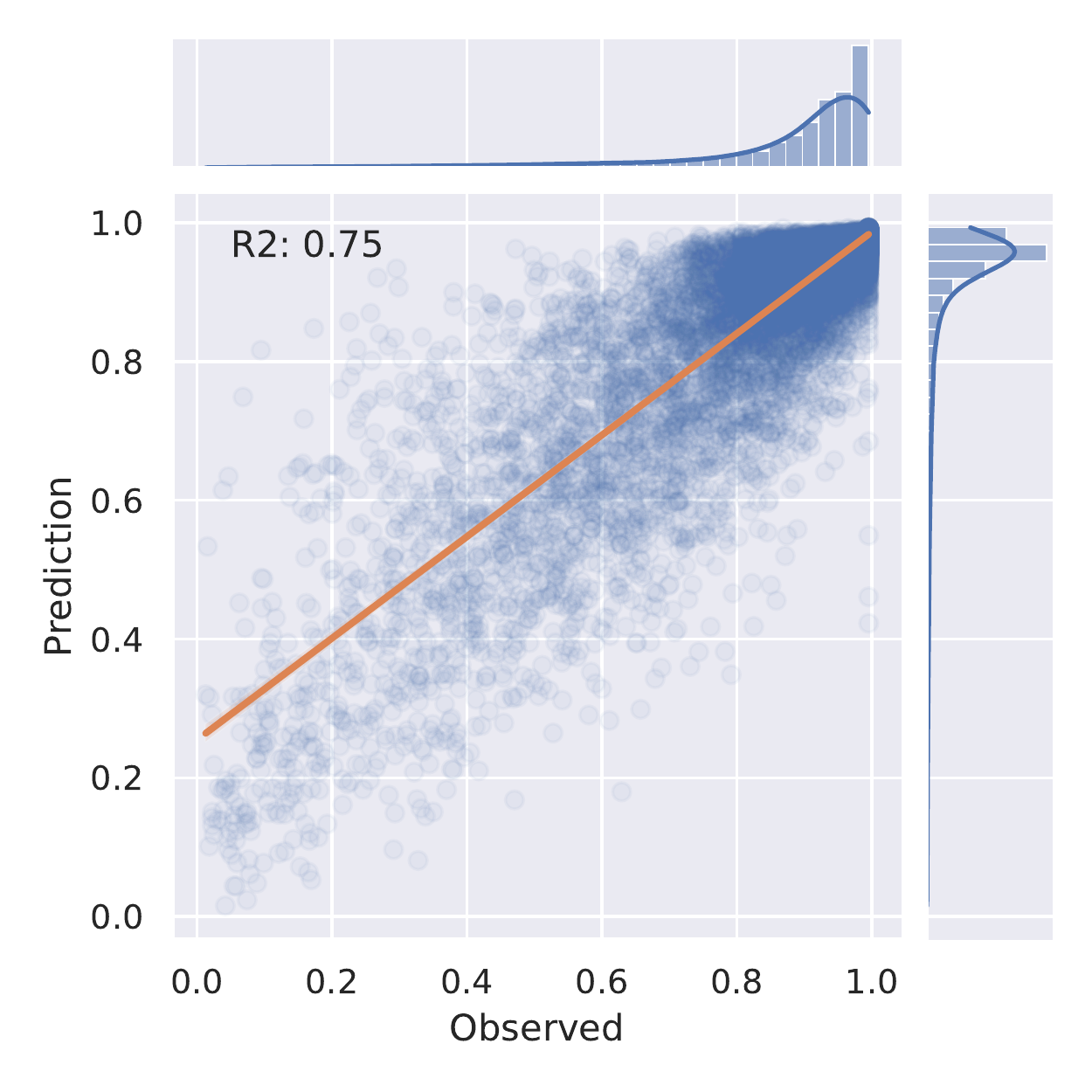}
    \end{subfigure}
    \caption{Bayesian factor model fit to the GDSC dataset. \emph{Left}: examples of raw data (points) and fitted curves (lines). \emph{Right}: scatter plot comparing fitted vs observed points.}
    \label{fig:gdsc-model-fit}
\end{figure}

\paragraph{Experiments} We study the potential of \pdbal to adaptively screen anti-cancer drugs in a retrospective experiment on GDSC. At each step, the algorithm is allowed to conduct a single trial of a drug tested against a cancer cell line. We consider coarse- and fine-grained settings for drug selection. In the coarse setting, the algorithm selects the drug and cell line and observes the responses at each of the seven doses; in the fine-grain setting, the algorithm additionally specifies a dose. Each setup corresponds to $n=2K$ and $n=14K$ possible experiments, respectively. Performance is evaluated as the error over each (cell line, drug, dose) when compared to what the underlying probabilistic model would learn from the full data.

Each algorithm was tested over the same set of 7 random seeds. Each run had a warm start, where one observation for each cell line and drug was randomly selected. The Bayesian model was updated after each selection with 10 parallel chains, each with 200 burn-in Gibbs sampling cycles. Each chain collected 10 samples with a thinning factor of 10. To further accelerate the evaluation, the \pdbal, \eig, and variance scores to select the next experiment were also parallelized across 10 processes. For \pdbal, the coarse-grained experiments took approximately 12 hours of computation and the fine-grained experiment took 48 hours on a cluster equipped with Intel ``Cascade Lake" CPUs. \eig scoring was slow due to numerical integration, making it infeasible to score the large number of possible experiments in the fined-grained setup. Therefore, it was only evaluated on the coarse-grained setup.

\vfill

\end{document}